\def\year{2021}\relax
\documentclass[letterpaper]{article} 
\usepackage{aaai21}  
\usepackage{times}  
\usepackage{helvet} 
\usepackage{courier}  
\usepackage[hyphens]{url}  
\usepackage{graphicx} 
\usepackage{natbib}  
\usepackage{caption} 
\newcommand*{\Resize}[2]{\resizebox{#1}{!}{$#2$}}%
\newcommand{\beginsupplement}{%
        \setcounter{table}{0}
        \renewcommand{\thetable}{S\arabic{table}}%
        \setcounter{algorithm}{0}
        \renewcommand{\thealgorithm}{S\arabic{algorithm}}%
        \setcounter{equation}{0}
        \renewcommand{\theequation}{S\arabic{equation}}%
        \setcounter{figure}{0}
        \renewcommand{\thefigure}{S\arabic{figure}}%
        \setcounter{section}{0}
        \renewcommand{\thesection}{S\arabic{section}}%
     }

\usepackage[implicit=false]{hyperref} 

\usepackage{amsmath}
\usepackage{amssymb}
\usepackage{algorithm}
\usepackage{algorithmic}
\usepackage{adjustbox}
\usepackage{amsthm}
\usepackage{multirow}
\usepackage{tensor}
\usepackage{multicol}
\usepackage{booktabs}
\usepackage{dblfloatfix}
\usepackage[labelsep=colon]{caption}
\newtheorem{theorem}{Theorem}

\newtheorem{lemma}{Lemma}
\newtheorem{proposition}{Proposition}
\newcommand{\norm}[1]{\left\lVert#1\right\rVert}
\newcommand*\samethanks[1][\value{footnote}]{\footnotemark[#1]}

\title{Liquid Time-constant Networks}
\author{
    Ramin Hasani,\textsuperscript{\rm 1,3}\thanks{Authors with equal contributions}
    Mathias Lechner,\textsuperscript{\rm 2}\samethanks[1]
    Alexander Amini,\textsuperscript{\rm 1}
    Daniela Rus,\textsuperscript{\rm 1}
    Radu Grosu\textsuperscript{\rm 3}
    \\
}
\affiliations{

    \textsuperscript{\rm 1} Massachusetts Institute of Technology (MIT)\\
    \textsuperscript{\rm 2} Institute of Science and Technology Austria (IST Austria)\\
    \textsuperscript{\rm 3} Technische Universita\"t Wien (TU Wien)\\

    rhasani@mit.edu, mathias.lechner@ist.ac.at, amini@mit.edu, rus@csail.mit.edu, radu.grosu@tuwien.ac.at

}

\begin{document}

\maketitle

\begin{abstract}
We introduce a new class of time-continuous recurrent neural network models. Instead of declaring a learning system's dynamics by implicit nonlinearities, we construct networks of linear first-order dynamical systems modulated via nonlinear interlinked gates. The resulting models represent dynamical systems with varying (i.e., \emph{liquid}) time-constants coupled to their hidden state, with outputs being computed by numerical differential equation solvers. These neural networks exhibit stable and bounded behavior, yield superior expressivity within the family of neural ordinary differential equations, and give rise to improved performance on time-series prediction tasks. To demonstrate these properties, we first take a theoretical approach to find bounds over their dynamics, and compute their expressive power by the \emph{trajectory length} measure in a latent trajectory space. We then conduct a series of time-series prediction experiments to manifest the approximation capability of Liquid Time-Constant Networks (LTCs) compared to classical and modern RNNs.\footnote[1]{Code and data are available at: \url{https://github.com/raminmh/liquid_time_constant_networks}
}
\end{abstract}

\section{Introduction}

Recurrent neural networks with continuous-time hidden states determined by ordinary differential equations (ODEs), are effective algorithms for modeling time series data that are ubiquitously used in medical, industrial and business settings. The state of a neural ODE, $\textbf{x}(t) \in \mathbb{R}^D$, is defined by the solution of this equation \cite{chen2018neural}: $d\textbf{x}(t)/dt = f(\textbf{x}(t), \textbf{I}(t), t, \theta)$, with a neural network $f$ parametrized by $\theta$. One can then compute the state using a numerical ODE solver, and train the network by performing reverse-mode automatic differentiation \cite{rumelhart1986learning}, either by gradient descent through the solver \cite{lechner2019designing}, or by considering the solver as a black-box \cite{chen2018neural,dupont2019augmented,gholami2019anode} and apply the \emph{adjoint method} \cite{pontryagin2018mathematical}. The open questions are: how expressive are neural ODEs in their current formalism, and can we improve their structure to enable richer representation learning and expressiveness? 

Rather than defining the derivatives of the hidden-state directly by a neural network $f$, one can determine a more stable continuous-time recurrent neural network (CT-RNN) by the following equation \cite{funahashi1993approximation}: 
$\frac{d\textbf{x}(t)}{dt} = - \frac{\textbf{x}(t)}{\tau} + f(\textbf{x}(t), \textbf{I}(t), t, \theta)$, in which the term $- \frac{\textbf{x}(t)}{\tau}$ assists the autonomous system to reach an equilibrium state with a time-constant $\tau$. $\textbf{x}(t)$ is the hidden state, $\textbf{I}(t)$ is the input, t represents time, and $f$ is parametrized by $\theta$. 

We propose an alternative formulation: let the hidden state flow of a network be declared by a system of linear ODEs of the form: $d\textbf{x}(t)/dt = - \textbf{x}(t)/\tau + \textbf{S}(t)$, and let $\textbf{S}(t) \in \mathbb{R}^M$ represent the following nonlinearity determined by $\textbf{S}(t) = f(\textbf{x}(t),\textbf{I}(t), t, \theta) (A - \textbf{x}(t))$, with parameters $\theta$ and $A$. Then, by plugging in $\textbf{S}$ into the hidden states equation, we get:

\begin{equation}
\label{eq:ltc_1}
\begin{split}
    \frac{d\textbf{x}(t)}{dt} = &- \Big[\frac{1}{\tau} + f(\textbf{x}(t), \textbf{I}(t),t, \theta)\Big] \textbf{x}(t) +\\  
    & f(\textbf{x}(t), \textbf{I}(t), t, \theta) A.
\end{split}
\end{equation}

Eq. \ref{eq:ltc_1} manifests a novel time-continuous RNN instance with several features and benefits:

\noindent\textbf{Liquid time-constant.} A neural network $f$ not only determines the derivative of the hidden state $\textbf{x}(t)$, but also serves as an input-dependent varying time-constant ($\tau_{sys} = \frac{\tau}{1+ \tau f(\textbf{x}(t), \textbf{I}(t), t, \theta)}$) for the learning system (Time constant is a parameter characterizing the speed and the coupling sensitivity of an ODE).This property enables single elements of the hidden state to identify specialized dynamical systems for input features arriving at each time-point. We refer to these models as \emph{liquid time-constant} recurrent neural networks (LTCs). LTCs can be implemented by an arbitrary choice of ODE solvers. In Section 2, we introduce a practical fixed-step ODE solver that simultaneously enjoys the stability of the implicit Euler and the computational efficiency of the explicit Euler methods. 

\noindent\textbf{Reverse-mode automatic differentiation of LTCs.} LTCs realize differentiable computational graphs. Similar to neural ODEs, they can be trained by variform of gradient-based optimization algorithms. We settle to trade memory for numerical precision during a backward-pass by using a vanilla backpropagation through-time algorithm to optimize LTCs instead of an adjoint-based optimization method \cite{pontryagin2018mathematical}. In Section 3, we motivate this choice thoroughly.

\noindent\textbf{Bounded dynamics - stability.} In Section 4, we show that the state and the time-constant of LTCs are bounded to a finite range. This property assures the stability of the output dynamics and is desirable when inputs to the system relentlessly increase.

\noindent\textbf{Superior expressivity.} In Section 5, we theoretically and quantitatively analyze the approximation capability of LTCs. We take a functional analysis approach to show the universality of LTCs. We then delve deeper into measuring their expressivity compared to other time-continuous models. We perform this by measuring the \emph{trajectory length} of activations of networks in a latent trajectory representation. 
Trajectory length was introduced as a measure of expressivity of feed-forward deep neural networks \cite{raghu2017expressive}. We extend these criteria to the family of continuous-time recurrent models.

\noindent\textbf{Time-series modeling.} In Section 6, we conduct a series of eleven time-series prediction experiments and compare the performance of modern RNNs to the time-continuous models. We observe improved performance on a majority of cases achieved by LTCs. 

\noindent \textbf{Why this specific formulation?}
There are two primary justifications for the choice of this particular representation:

\noindent I) LTC model is loosely related to the computational models of neural dynamics in small species, put together with synaptic transmission mechanisms \cite{hasani2020natural}. The dynamics of non-spiking neurons' potential, $\textbf{v}(t)$, can be written as a system of linear ODEs of the form \cite{lapicque1907recherches,Koch98}: $d\textbf{v}/dt = - g_l \textbf{v}(t) + \textbf{S}(t)$, where $\textbf{S}$ is the sum of all synaptic inputs to the cell from presynaptic sources, and $g_{l}$ is a leakage conductance.

All synaptic currents to the cell can be approximated in steady-state by the following nonlinearity \cite{Koch98,wicks1996dynamic}: $\textbf{S}(t) = f(\textbf{v}(t), \textbf{I}(t)), (A - \textbf{v}(t))$, where $f(.)$ is a sigmoidal nonlinearity depending on the state of all neurons, $\textbf{v}(t)$ which are presynaptic to the current cell, and external inputs to the cell, $I(t)$. By plugging in these two equations, we obtain an equation similar to Eq. \ref{eq:ltc_1}. LTCs are inspired by this foundation.

\noindent II) Eq. \ref{eq:ltc_1} might resemble that of the famous Dynamic Causal Models (DCMs) \cite{friston2003dynamic} with a Bilinear dynamical system approximation \cite{penny2005bilinear}. DCMs are formulated by taking a second-order approximation (Bilinear) of the dynamical system $d\textbf{x}/dt = F(\textbf{x}(t), \textbf{I}(t), \theta)$, that would result in the following format \cite{friston2003dynamic}: $d\textbf{x}/dt = (A +\textbf{I}(t) B) \textbf{x}(t) + C \textbf{I}(t)$ with $A = \frac{dF}{d\textbf{x}}$, $B = \frac{dF^2}{d\textbf{x}(t) d\textbf{I}(t)}$, $C = \frac{dF}{d\textbf{I}(t)}$. DCM and bilinear dynamical systems have shown promise in learning to capture complex fMRI time-series signals. LTCs are introduced as variants of continuous-time (CT) models that are loosely inspired by biology, show great expressivity, stability, and performance in modeling time series.


\section{LTCs forward-pass by a fused ODE solvers}

Solving Eq. \ref{eq:ltc_1} analytically, is non-trivial due to the nonlinearity of the LTC semantics. The state of the system of ODEs, however, at any time point $T$, can be computed by a numerical ODE solver that simulates the system starting from a trajectory $x(0)$, to $x(T)$. An ODE solver breaks down the continuous simulation interval $[0,T]$ to a temporal discretization, $[t_0,t_1,\dots t_n]$. As a result, a solver's step involves only the update of the neuronal states from $t_i$ to $t_{i+1}$.

LTCs’ ODE realizes a system of stiff equations \cite{Press2007}. This type of ODE requires an exponential number of discretization steps when simulated with a Runge-Kutta (RK) based integrator. Consequently, ODE solvers based on RK, such as Dormand–Prince (default in torchdiffeq \cite{chen2018neural}), are not suitable for LTCs. Therefore, We design a new ODE solver that fuses the explicit and the implicit Euler methods \cite{Press2007}. This choice of discretization method results in achieving stability for an implicit update equation. To this end, the \emph{Fused Solver} numerically unrolls a given dynamical system of the form $dx/dt = f(x)$ by:
\begin{equation}\label{eq:hybrid}
x(t_{i+1}) =  x(t_i) + \Delta t f(x(t_i), x(t_{i+1})).
\end{equation}

\begin{algorithm}[t]
\caption{LTC update by fused ODE Solver}
\label{algorithm:chap3_LTC_Cell_Update}
\begin{algorithmic}
\STATE \textbf{Parameters:} $\theta =$ \{$\tau^{(N \times 1)}$ = time-constant, $\gamma^{(M \times N)}$ = weights, $\gamma_{r}^{(N \times N)}$ = recurrent weights, $\mu^{(N \times 1)}$ = biases\}, $A^{(N \times 1)}$ = bias vector, $L=$ Number of unfolding steps, $\Delta t =$ step size, $N = $ Number of neurons, 
\STATE \textbf{Inputs:} $M$-dimensional Input $\textbf{I}(t)$ of length $T$, $\textbf{x}(0)$
\STATE \textbf{Output:} Next LTC neural state $\textbf{x}_{t+\Delta t}$
\STATE \textbf{Function:} FusedStep($\textbf{x}(t)$, $\textbf{I}(t)$, $\Delta t$, $\theta$)
\STATE $\textbf{x}(t+\Delta t)^{(N \times T)} = \frac{\textbf{x}(t)~+~ \Delta t f(\textbf{x}(t), \textbf{I}(t), t, \theta) \odot A} {1 + \Delta t \big( 1/\tau + f(\textbf{x}(t), \textbf{I}(t), t, \theta)\big)}$
\STATE $\triangleright$~~$f(.)$, and all divisions are applied element-wise.
\STATE $\triangleright$~~$\odot$ is the Hadamard product.
\STATE \textbf{end Function}
\STATE $\textbf{x}_{t+\Delta t} = \textbf{x}(t)$
\FOR{$i = 1 \dots L$}
\STATE $\textbf{x}_{t+\Delta t} =$ FusedStep($\textbf{x}(t)$, $\textbf{I}(t)$, $\Delta t$, $\theta$)
\ENDFOR
\RETURN $\textbf{x}_{t+\Delta t}$
\end{algorithmic}
\end{algorithm}

In particular, we replace only the $x(t_i)$ that occur linearly in $f$ by $x(t_{i+1})$. As a result, Eq \ref{eq:hybrid} can be solved for $x(t_{i+1})$, symbolically. Applying the Fused solver to the LTC representation, and solving it for $\textbf{x}(t+\Delta t)$, we get:

\begin{equation}
    \label{eq:ltc_fused_1}
 \textbf{x}(t+\Delta t)= \frac{\textbf{x}(t)~+~ \Delta t f(\textbf{x}(t),\textbf{I}(t), t, \theta) A} {1 + \Delta t \big( 1/\tau + f(\textbf{x}(t), \textbf{I}(t), t, \theta)\big)}.
\end{equation}




Eq. \ref{eq:ltc_fused_1} computes one update state for an LTC network. Correspondingly, Algorithm \ref{algorithm:chap3_LTC_Cell_Update} shows how to implement an LTC network, given a parameter space $\theta$. $f$ is assumed to have an arbitrary activation function (e.q. for a $tanh$ nonlinearity $f = \tanh(\gamma_r \textbf{x} + \gamma \textbf{I} + \mu)$).
The computational complexity of the algorithm for an input sequence of length $T$ is $O(L \times T)$, where $L$ is the number of discretization steps. Intuitively, a dense version of an LTC network with $N$ neurons, and a dense version of a long short-term memory (LSTM) \cite{hochreiter1997long} network with $N$ cells, would be of the same complexity.

\section{Training LTC networks by BPTT}

Neural ODEs were suggested to be trained by a constant memory cost for each layer in a neural network $f$ by applying the adjoint sensitivity method to perform reverse-mode automatic differentiation \cite{chen2018neural}. The adjoint method, however, comes with numerical errors when running in reverse mode. This phenomenon happens because the adjoint method forgets the forward-time computational trajectories, which was repeatedly denoted by the community \cite{gholami2019anode,zhuang2020ada}.


On the contrary, direct backpropagation through time (BPTT) trades memory for accurate recovery of the forward-pass during the reverse mode integration \cite{zhuang2020ada}. Thus, we set out to design a vanilla BPTT algorithm to maintain a highly accurate backward-pass integration through the solver. For this purpose, a given ODE solver's output (a vector of neural states), can be recursively folded to build an RNN and then apply the learning algorithm described in Algorithm \ref{algorithm:chap3_training} to train the system. Algorithm \ref{algorithm:chap3_training} uses a vanilla stochastic gradient descent (SGD). One can substitute this with a more performant variant of the SGD, such as Adam \cite{kingma2014adam}, which we use in our experiments.

\begin{algorithm}[t]
\caption{Training LTC by BPTT}
\label{algorithm:chap3_training}
\begin{algorithmic}
\STATE \textbf{Inputs:} Dataset of traces $[I(t),y(t)]$ of length $T$, RNNcell $= f(I,x)$
\STATE \textbf{Parameter:} Loss func $L(\theta)$, initial param $\theta_0$, learning rate $\alpha$, Output w $= W_{out}$, and bias $= b_{out}$
\FOR{$i = 1 \dots$ number of training steps}
\STATE $(I_b$,$y_b) =$ Sample training batch,~~~~~$x := x_{t_0} \sim p(x_{t_0})$ 
\FOR{$j = 1 \dots T$}
\STATE $x = f(I(t),x)$,~~~$\hat{y}(t) = W_{out}.x + b_{out}$,~~~$L_{total} = \sum_{j=1}^{T} L(y_j(t), \hat{y}_j(t))$,~~~$\nabla L(\theta) = \frac{\partial L_{tot}}{\partial \theta}$
\STATE $\theta = \theta - \alpha \nabla L(\theta)$
\ENDFOR
\ENDFOR
\RETURN $\theta$
\end{algorithmic}
\end{algorithm}

\begin{table}[t]
\footnotesize
    \centering
    \caption{Complexity of the vanilla BPTT compared to the adjoint method, for a single layer neural network $f$}
    \vspace{-3mm}
    \begin{tabular}{c|cc}
    \hline
         & \textbf{Vanilla BPTT} & \textbf{Adjoint} \\
         \hline
        Time & $O(L \times T \times 2)$  & $O((L_f + L_b) \times T)$ \\
        Memory & $O(L \times T)$ & $\textbf{O(1)}$ \\
        Depth & $O(L)$ & $O(L_b)$ \\
        FWD acc  & High & High\\ 
        BWD acc & \textbf{High} & Low \\
        \hline
    \end{tabular}
    \caption*{\tiny \textbf{Note:} $L =$ number of discretization steps, $L_f =$ L during forward-pass. $L_b =$ L during backward-pass. $T = $ length of sequence, Depth = computational graph depth.}
    \label{tab:complexity}
    \vspace{-9mm}
\end{table}

\noindent\textbf{Complexity.} Table \ref{tab:complexity} summarizes the complexity of our vanilla BPTT algorithm compared to an adjoint method. We achieve a high degree of accuracy on both forward and backward integration trajectories, with similar computational complexity, at large memory costs. 

\section{Bounds on $\tau$ and neural state of LTCs}
LTCs are represented by an ODE which varies its time-constant based on inputs. It is therefore important to see if LTCs stay stable for unbounded arriving inputs \cite{hasani2019response,lechner2020gershgorin}. In this section, we prove that the time-constant and the state of LTC neurons are bounded to a finite range, as described in Theorems \ref{lemma_tau} and \ref{lemma_neuralstate}, respectively.
\begin{theorem}
\label{lemma_tau}
Let $x_i$ denote the state of a neuron $i$ within an LTC network identified by Eq. \ref{eq:ltc_1}, and let neuron $i$ receive $M$ incoming connections. Then, the time-constant of the neuron, $\tau_{sys_i}$, is bounded to the following range:
\begin{equation}
\label{eq:tau}
 \tau_i/(1+\tau_i W_i) \leq \tau_{sys_i} \leq \tau_i,
\end{equation}
\end{theorem}

The proof is provided in Appendix. It is constructed based on bounded, monotonically increasing sigmoidal nonlinearity for neural network $f$ and its replacement in the LTC network dynamics. A stable varying time-constant significantly enhances the expressivity of this form of time-continuous RNNs, as we discover more formally in Section 5.

\begin{theorem}
\label{lemma_neuralstate}
Let $x_i$ denote the state of a neuron $i$ within an LTC, identified by Eq. \ref{eq:ltc_1}, and let neuron $i$ receive $M$ incoming connections. Then, the hidden state of any neuron $i$, on a finite interval $Int \in[0,T]$, is bounded as follows:
\begin{equation}
\label{eq:vbound}
 {min}(0, A_{i}^{min}) \leq x_i(t) \leq {max}(0, A_{i}^{max}),
\end{equation}
\end{theorem}
The proof is given in Appendix. It is constructed based on the sign of the LTC's equation's compartments, and an approximation of the ODE model by an explicit Euler discretization. Theorem \ref{lemma_neuralstate} illustrates a desired property of LTCs, namely \emph{state stability} which guarantees that the outputs of LTCs never explode even if their inputs grow to infinity. Next we discuss the expressive power of LTCs compared to the family of time-continuous models, such as CT-RNNs and neural ordinary differential equations \cite{chen2018neural,rubanova2019latent}. 

\begin{figure}[t]
\centering
\includegraphics[width=0.4\textwidth]{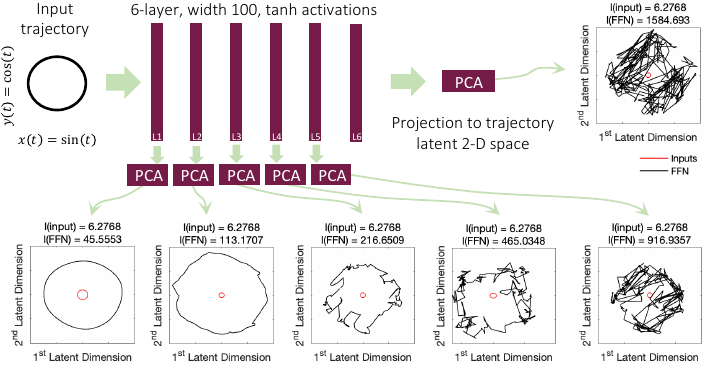}
\caption{Trajectory's latent space becomes more complex as the input passes through hidden layers.}
\label{fig:traj_ffn}
\end{figure}
\section{On the expressive power of LTCs}
Understanding how the structural properties of neural networks determine which functions they can compute is known as the expressivity problem. The very early attempts on measuring expressivity of neural nets include the theoretical studies based on functional analysis. They show that neural networks with three-layers can approximate any finite set of continuous mapping with any precision. This is known as the \emph{universal approximation theorem} \cite{hornik1989multilayer,funahashi1989approximate,cybenko1989approximation}. Universality was extended to standard RNNs \cite{funahashi1989approximate} and even continuous-time RNNs \cite{funahashi1993approximation}. By careful considerations, we can also show that LTCs are also universal approximators. 

\begin{figure*}[b]
 \centering
 \includegraphics[width=0.9\textwidth]{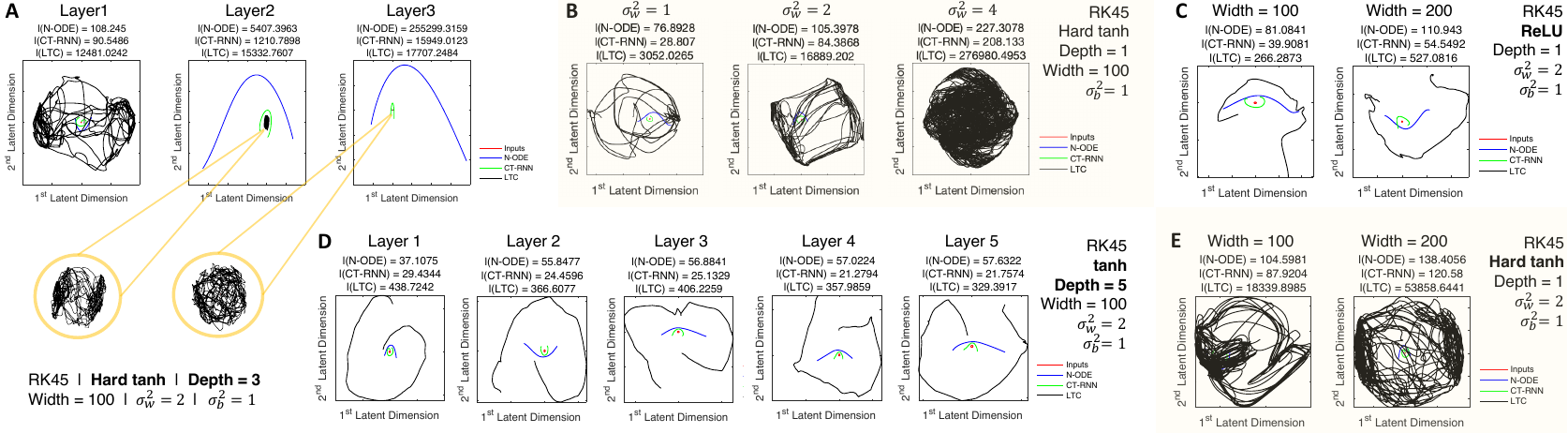}
 \caption{Trajectory length deformation A) in network layers with \url{Hard-tanh} activations, B) as a function of the weight distribution scaling factor, C) as a function of network width (\url{ReLU}), D) in network layers with \url{logistic-sigmoid} activations and E) as a function of width (\url{Hard-tanh}).}
 \label{fig:traj_glimpse}
\end{figure*}

\begin{theorem}
\label{univ_approx_ltc}
Let $\textbf{x} \in \mathbb{R}^n$, $S\subset\mathbb{R}^n$ and $\dot{\textbf{x}} = F(\textbf{x})$ be an autonomous ODE with $F:S\rightarrow\mathbb{R}^n$ a $C^1$-mapping on $S$. Let $D$ denote a compact subset of $S$ and assume that the simulation of the system is bounded in the interval $I = [0,~T]$. Then, for a positive $\epsilon$, there exist an LTC network with $N$ hidden units, $n$ output units, and an output internal state $\textbf{u}(t)$, described by Eq. \ref{eq:ltc_1}, such that for any rollout $\{\textbf{x}(t)|t\,{\in}\,I\}$ of the system with initial value $x(0)\,{\in}\,D$, and a proper network initialization, 
\begin{equation}
{max}_{t\,{\in}\,I} |\textbf{x}(t)\,{-}\,\textbf{u}(t)|\,{<}\,\epsilon
\end{equation}
\end{theorem}

The main idea of the proof is to define an $n$-dimensional dynamical system and place it into a higher dimensional system. The second system is an LTC. The fundamental difference of the proof of LTC's universality to
that of CT-RNNs \cite{funahashi1993approximation} lies in the distinction of the semantics of both systems where the LTC network contains a nonlinear input-dependent term in its time-constant module which makes parts of the proof non-trivial. 

The universal approximation theorem broadly explores the expressive power of a neural network model. The theorem however, does not provide us with a foundational measure on where the separation is between different neural network architectures. Therefore, a more rigorous measure of expressivity is demanded to compare models, specifically those networks specialized in spatiotemporal data processing, such as LTCs. The advances made on defining measures for the expressivity of static deep learning models \cite{pascanu2013number,montufar2014number,eldan2016power,poole2016exponential,raghu2017expressive} could presumably help measure the expressivity of time-continuous models, both theoretically and quantitatively, which we explore in the next section.

\subsection{Measuring expressivity by trajectory length}
A measure of expressivity has to take into account what degrees of complexity a learning system can compute, given the network's capacity (depth, width, type, and weights configuration). A unifying expressivity measure of static deep networks is the \emph{trajectory length} introduced in \cite{raghu2017expressive}. In this context, one evaluates how a deep model transforms a given input trajectory (e.g., a circular 2-dimensional input) into a more complex pattern, progressively.

\begin{table}[t]
\tiny
    \centering
    \caption{\textbf{Computational depth of models}}
    \begin{tabular}{l|ccc}
        \toprule
        & \multicolumn{3}{c}{\textbf{Computational Depth}}\\
        Activations &  Neural ODE  & CT-RNN & \textbf{LTC}   \\
        \midrule
        \url{tanh} & 0.56 $\pm$ 0.016 & 4.13 $\pm$ 2.19 & 9.19 $\pm$ 2.92 \\
        \url{sigmoid}  & 0.56 $\pm$ 0.00 & 5.33 $\pm$ 3.76 & 7.00 $\pm$ 5.36\\
        \url{ReLU} & 1.29 $\pm$ 0.10 & 4.31 $\pm$ 2.05 & 56.9 $\pm$ 9.03\\
        \url{Hard-tanh} & 0.61 $\pm$ 0.02 &  4.05 $\pm$ 2.17 & 81.01 $\pm$ 10.05\\
        \bottomrule
    \end{tabular}
    \caption*{\tiny \textbf{Note:} $\#$ of tries = 100, input samples' $\Delta t = 0.01$, $T = 100$ sequence length. $\#$ of layers = 1, width = 100, $\sigma^{2}_{w} = 2$, $\sigma^{2}_{b} = 1$.}
    \label{tab:depth}
    \vspace{-8mm}
\end{table}

We can then perform principle component analysis (PCA) over the obtained network's activations. Subsequently, we measure the length of the output trajectory in a 2-dimensional latent space, to uncover its relative complexity (see Fig. \ref{fig:traj_ffn}). The trajectory length is defined as the \emph{arc length} of a given trajectory $I(t)$, (e.g. a circle in 2D space) \cite{raghu2017expressive}: $l(I(t)) = \int_t \norm{dI(t)/dt} dt$. By establishing a lower-bound for the growth of the trajectory length, one can set a barrier between networks of shallow and deep architectures, regardless of any assumptions on the network's weight configuration \cite{raghu2017expressive}, unlike many other measures of expressivity \cite{pascanu2013number,montufar2014number,serra2017bounding,gabrie2018entropy,hanin2018start,hanin2019complexity,lee2019towards}.
\begin{figure*}[t]
 \centering
 \includegraphics[width=0.8\textwidth]{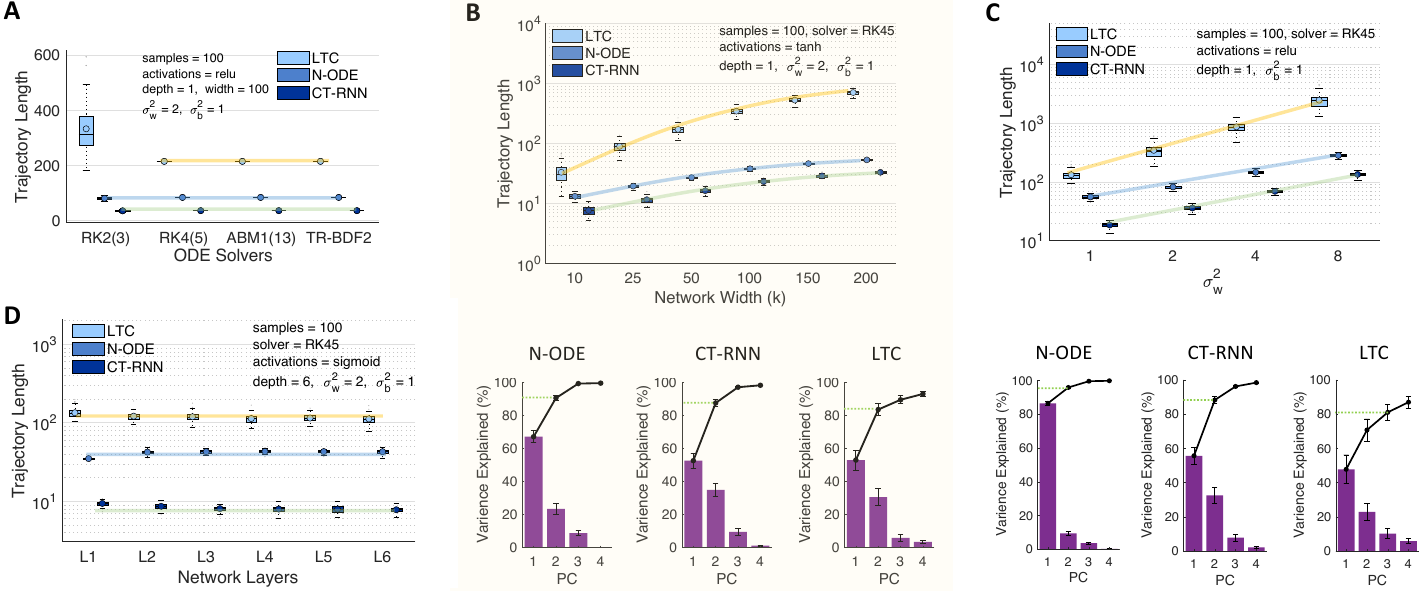}
\caption{Dependencies of the trajectory length measure. A) trajectory length vs different solvers (variable-step solvers). RK2(3): Bogacki-Shampine Runge-Kutta (2,3) \protect\cite{bogacki19893}. RK4(5): Dormand-Prince explicit RK (4,5) \protect\cite{dormand1980family}. ABM1(13): Adams-Bashforth-Moulton \protect\cite{shampine1975computer}. TR-BDF2: implicit RK solver with 1st stage trapezoidal rule and a 2nd stage backward differentiation \protect\cite{hosea1996analysis}. 
B) Top: trajectory length vs network width. Bottom: Variance-explained of principle components (purple bars) and their cumulative values (solid black line). C) Trajectory length vs weights distribution variance. D) trajectory length vs layers. (More results in the supplements)
}
 \label{fig:traj_results}
\end{figure*}
We set out to extend the trajectory-space analysis of static networks to time-continuous (TC) models, and to lower-bound the trajectory length to compare models' expressivity. To this end, we designed instances of Neural ODEs, CT-RNNs and LTCs with shared $f$. The networks were initialized by weights $ \sim \mathcal{N}(0,\sigma^{2}_{w} / k)$, and biases $\sim \mathcal{N}(0,\sigma^{2}_{b})$. We then perform forward-pass simulations by using different types of ODE solvers, for arbitrary weight profiles, while exposing the networks to a circular input trajectory $I(t)= \{I_1(t) = \sin(t), I_2(t) = \cos(t)\}$, for $t \in [0, 2\pi]$. By looking at the first two principle components (with an average variance-explained of over $80\%$) of hidden layers' activations, we observed consistently more complex trajectories for LTCs. Fig. \ref{fig:traj_glimpse} gives a glimpse of our empirical observations. All networks are implemented by the Dormand-Prince explicit Runge-Kutta(4,5) solver \cite{dormand1980family} with a variable step size. We had the following \textbf{observations}: 
\textbf{I)} Exponential growth of the trajectory length of Neural ODEs and CT-RNNs with \url{Hard-tanh} and \url{ReLU} activations (Fig. \ref{fig:traj_glimpse}A) and unchanged shape of their latent space regardless of their weight profile.
\textbf{II)} LTCs show a slower growth-rate of the trajectory length when designed by \url{Hard-tanh} and \url{ReLU}, with the compromise of realizing great levels of complexity (Fig. \ref{fig:traj_glimpse}A, \ref{fig:traj_glimpse}C and \ref{fig:traj_glimpse}E). \textbf{III)} Apart from multi-layer time-continuous models built by \url{Hard-tanh} and \url{ReLU} activations, in all cases, we observed a longer and a more complex latent space behavior for the LTC networks (Fig. \ref{fig:traj_glimpse}B to \ref{fig:traj_glimpse}E). \textbf{IV)} Unlike static deep networks (Fig. \ref{fig:traj_ffn}), we witnessed that the trajectory length does not grow by depth in multi-layer continuous-time networks realized by \url{tanh} and \url{sigmoid} (Fig. \ref{fig:traj_glimpse}D). \textbf{V)} conclusively, we observed that the trajectory length in TC models varies by a model's activations, weight and bias distributions variance, width and depth. We presented this more systematically in Fig. \ref{fig:traj_results}. \textbf{VI)} Trajectory length grows linearly with a network's width (Fig. \ref{fig:traj_results}B - Notice the logarithmic growth of the  curves in the log-scale Y-axis). \textbf{VII)} The growth is considerably faster as the variance grows (Fig. \ref{fig:traj_results}C). \textbf{VIII)} Trajectory length is reluctant to the choice of ODE solver (Fig. \ref{fig:traj_results}A). \textbf{IX)} Activation functions diversify the complex patterns explored by the TC system, where \url{ReLU} and \url{Hard-tanh} networks demonstrate higher degrees of complexity for LTCs. A key reason is the presence of recurrent links between each layer's cells. 
\textbf{Definition of Computational Depth (\textit{L}). } For one hidden layer of $f$ in a time-continuous network, $L$ is the average number of integration steps taken by the solver for each incoming input sample. Note that for an $f$ with $n$ layers we define the total depth as $n \times L$. These observations have led us to formulate Lower bounds for the growth of the trajectory length of continuous-time networks.

\begin{theorem}
\label{theorem:neural_ode}
 \emph{Trajectory Length growth Bounds for Neural ODEs and CT-RNNs.} Let $dx/dt = f_{n,k}(\textbf{x}(t),\textbf{I}(t), \theta)$ with $\theta = \{W, b\}$, represent a Neural ODE and $\frac{d\textbf{x}(t)}{dt} = - \frac{\textbf{x}(t)}{\tau} +  f_{n,k}(\textbf{x}(t),\textbf{I}(t), \theta)$ with $\theta = \{W, b, \tau \}$ a CT-RNN. $f$ is randomly weighted with\emph{\url{Hard-tanh}} activations. Let $\textbf{I}(t)$ be a 2D input trajectory, with its progressive points (i.e. $I(t+\delta t)$) having a perpendicular component to $\textbf{I}(t)$ for all $\delta t$, with $L =$ number of solver-steps. Then, by defining the projection of the first two principle components' scores of the hidden states over each other, as the 2D \emph{latent trajectory space} of a layer $d$, $z^{(d)}(\textbf{I}(t)) = z^{(d)}(t)$, for Neural ODE and CT-RNNs respectively, we have:
\begin{gather}
\mathbb{E}\Bigg[ l(z^{(d)}(t))\Bigg] 
    \geq
    O\Bigg(\frac{\sigma_w\sqrt{k}}{\sqrt{ \sigma^{2}_{w} + \sigma^{2}_{b} + k \sqrt{\sigma^{2}_{w} + \sigma^{2}_{b}}}}\Bigg)^{d \times L} l(I(t)),
    \\
\mathbb{E}\Bigg[ l(z^{(d)}(t))\Bigg] 
    \geq
    O\Bigg(\frac{(\sigma_w-\sigma_b)\sqrt{k}}{\sqrt{ \sigma^{2}_{w} + \sigma^{2}_{b} + k \sqrt{\sigma^{2}_{w} + \sigma^{2}_{b}}}}\Bigg)^{d \times L} l(I(t)).
\end{gather}
\end{theorem}

The proof is provided in Appendix. It follows similar steps as \cite{raghu2017expressive} on the trajectory length bounds established for deep networks with piecewise linear activations, with careful considerations due to the continuous-time setup. The proof is constructed such that we formulate a recurrence between the norm of the hidden state gradient in layer $d+1$, $\norm{dz/dt^{(d+1)}}$, in principle components domain, and the expectation of the norm of the right-hand-side of the differential equations of neural ODEs and CT-RNNs. We then roll back the recurrence to reach the inputs. 

Note that to reduced the complexity of the problem, we only bounded the orthogonal components of the hidden state image $\norm{dz/dt^{(d+1)}_{\bot}}$, and therefore we have the assumption on input $I(t)$, in the Theorem's statement \cite{raghu2017expressive}. Next, we find a lower-bound for the LTC networks.

\begin{theorem}
\label{theorem:LTC}
 \emph{Growth Rate of LTC's Trajectory Length.} Let Eq. \ref{eq:ltc_1} determine an LTC with $\theta = \{W, b, \tau, A \}$. With the same conditions on $f$ and $I(t)$, as in Theorem \ref{theorem:neural_ode}, we have:
 
\begin{equation}
\label{eq:bound_ltc}
\begin{split}
\mathbb{E}\Bigg[ l(z^{(d)}(t))\Bigg] 
     \geq 
   O\Bigg(&\Big(\frac{\sigma_w\sqrt{k}}{\sqrt{ \sigma^{2}_{w} + \sigma^{2}_{b} + k \sqrt{\sigma^{2}_{w} + \sigma^{2}_{b}}}}\Big)^{d \times L} \times\\
   & \Big(\sigma_{w} + \frac{\norm{z^{(d)}}}{\min(\delta t, L)}\Big)\Bigg) l(I(t)).\end{split}
\end{equation}
\end{theorem}

\begin{table*}[t!]
\footnotesize
 \centering
 \caption{\textbf{Time series prediction} Mean and standard deviation, n=5}
 \begin{tabular}{l c |c|c|c|c|c}
 \toprule
Dataset & Metric  & LSTM   & CT-RNN & Neural ODE  & CT-GRU & LTC (ours) \\\hline
  Gesture &(accuracy) & 64.57\% $\pm$ 0.59 & 59.01\% $\pm$ 1.22 & 46.97\% $\pm$ 3.03 & 68.31\% $\pm$ 1.78 & \textbf{69.55\% $\pm$ 1.13}  \\
  Occupancy & (accuracy) & 93.18\% $\pm$ 1.66 & 94.54\% $\pm$ 0.54  & 90.15\% $\pm$ 1.71 & 91.44\% $\pm$ 1.67 & \textbf{94.63\% $\pm$ 0.17} \\
  Activity recognition & (accuracy) & 95.85\% $\pm$ 0.29 & 95.73\% $\pm$ 0.47 & \textbf{97.26}\% $\pm$ 0.10 & 96.16\% $\pm$ 0.39& 95.67\% $\pm$ 0.575   \\
  Sequential MNIST & (accuracy)  & \textbf{98.41}\% $\pm$ 0.12 & 96.73\% $\pm$ 0.19 & 97.61\% $\pm$ 0.14 & 98.27\% $\pm$ 0.14 & 97.57\% $\pm$ 0.18\\
  Traffic  & (squared error) & 0.169 $\pm$ 0.004 & 0.224 $\pm$ 0.008  & 1.512 $\pm$ 0.179 & 0.389 $\pm$ 0.076 & \textbf{0.099} $\pm$ 0.0095 \\
  Power & (squared-error)  & 0.628 $\pm$ 0.003 & 0.742 $\pm$ 0.005 & 1.254 $\pm$ 0.149 & \textbf{0.586} $\pm$ 0.003 & 0.642 $\pm$ 0.021\\
  Ozone & (F1-score) & 0.284 $\pm$ 0.025 & 0.236 $\pm$ 0.011 & 0.168 $\pm$ 0.006  & 0.260 $\pm$ 0.024& \textbf{0.302} $\pm$ 0.0155  \\
 \bottomrule
 \end{tabular}
 \label{tab:res_32}
\end{table*}

\begin{table}[t]
    \centering
    \caption{Person activity, 1st setting - n=5}
    \vspace{-3mm}
    \label{tab:per-time-point_classification}
    \begin{tabular}{lc}
    \toprule
         \textbf{Algorithm} & \textbf{Accuracy} \\
         \midrule
        LSTM & 83.59$\% \pm$ 0.40 \\
        CT-RNN & 81.54$\% \pm$ 0.33 \\
        Latent ODE & 76.48$\% \pm$ 0.56 \\
        CT-GRU & 85.27$\% \pm$ 0.39 \\
        LTC (ours) & \textbf{85.48$\% \pm$ 0.40}\\
      \bottomrule
    \end{tabular}
    \vspace{-5mm}
\end{table} 

\begin{table}[t] 
    \centering
    \caption{Person activity, 2nd setting}
    \label{tab:per-time-point_classification_2}
    \begin{tabular}{lc}
    \toprule
         \textbf{Algorithm} & \textbf{Accuracy} \\
         \midrule
        RNN $\Delta_t$ $^{*}$  & 0.797$\pm$ 0.003 \\
        RNN-Decay$^{*}$  & 0.800$\pm$ 0.010 \\
        RNN GRU-D$^{*}$  & 0.806$\pm$ 0.007 \\
        RNN-VAE$^{*}$  & 0.343$\pm$ 0.040 \\
        Latent ODE (D enc.)$^{*}$ & 0.835$\pm$ 0.010 \\
        ODE-RNN $^{*}$ & 0.829 $\pm$ 0.016 \\
        Latent ODE(C enc.)${^*}$ & 0.846 $\pm$ 0.013 \\
        LTC (ours) & \textbf{0.882 $\pm$ 0.005}\\
      \bottomrule
    \end{tabular}
    \caption*{\small \textbf{Note:} Accuracy for algorithms indicated by $*$, are taken directly from \cite{rubanova2019latent}. RNN $\Delta_t$ = classic RNN + input delays \cite{rubanova2019latent}. RNN-Decay =  RNN with exponential decay on the hidden
states \protect\cite{mozer2017discrete}. GRU-D = gated recurrent unit + exponential decay + input imputation \protect\cite{che2018recurrent}. D-enc. = RNN encoder \cite{rubanova2019latent}. C-enc = ODE encoder \cite{rubanova2019latent}. n=5}
\end{table}

The proof is provided in Appendix. A rough outline: we construct the recurrence between the norm of the hidden state gradients and the components of the right-hand-side of LTC separately which progressively build up the bound.

\subsection{Discussion of the theoretical bounds}
\textbf{I)} As expected, the bound for the Neural ODEs is very similar to that of an $n$ layer static deep network with the exception of the exponential dependencies to the number of solver-steps, $L$. \textbf{II)} The bound for CT-RNNs suggests their shorter trajectory length compared to neural ODEs, according to the base of the exponent. This results consistently matches our experiments presented in Figs. \ref{fig:traj_glimpse} and \ref{fig:traj_results}. \textbf{III)} Fig. \ref{fig:traj_glimpse}B and Fig. \ref{fig:traj_results}C show a faster-than-linear growth for LTC's trajectory length as a function of weight distribution variance. This is confirmed by LTC's lower bound shown in Eq. \ref{eq:bound_ltc}. \textbf{IV)} LTC's lower bound also depicts the linear growth of the trajectory length with the width, $k$, which validates the results presented in \ref{fig:traj_results}B. \textbf{V)} Given the computational depth of the models $L$ in Table \ref{tab:depth} for \url{Hard-tanh} activations, the computed lower bound for neural ODEs, CT-RNNs and LTCs justify a longer trajectory length of LTC networks in the experiments of Section 5. Next, we assess the expressive power of LTCs in a set of real-life time-series prediction tasks.

\section{Experimental Evaluation}
\textbf{6.1~~Time series predictions.} We evaluated the performance of LTCs realized by the proposed Fused ODE solver against the state-of-the-art discretized RNNs, LSTMs \cite{hochreiter1997long}, CT-RNNs (ODE-RNNs) \cite{funahashi1993approximation,rubanova2019latent}, continuous-time gated recurrent units (CT-GRUs) \cite{mozer2017discrete}, and Neural ODEs constructed by a $4^{th}$ order Runge-Kutta solver as suggested in \cite{chen2018neural}, in a series of diverse real-life supervised learning tasks. The results are summarized in Table \ref{tab:res_32}. The experimental setup are provided in Appendix. We observed between $5\%$ to $70\%$ performance improvement achieved by the LTCs compared to other RNN models in four out of seven experiments and comparable performance in the other three (see Table \ref{tab:res_32}). 

\noindent\textbf{6.2~~Person activity dataset.} We use the "Human Activity" dataset described in \cite{rubanova2019latent} in two distinct frameworks. The dataset consists of 6554 sequences of activity of humans (e.g. lying, walking, sitting), with a period of 211 ms. we designed two experimental frameworks to evaluate models' performance. In the \textit{1st Setting}, the baselines are the models described before, and the input representations are unchanged (details in Appendix). LTCs outperform all models and in particular CT-RNNs and neural ODEs with a large margin as shown in Table \ref{tab:per-time-point_classification}. Note that the CT-RNN architecture is equivalent to the ODE-RNN described in \cite{rubanova2019latent}, with the difference of having a state damping factor $\tau$.

In the \textit{2nd Setting}, we carefully set up the experiment to match the modifications made by \cite{rubanova2019latent} (See supplements), to obtain a fair comparison between LTCs and a more diverse set of RNN variants discussed in \cite{rubanova2019latent}. LTCs show superior performance with a high margin compared to other models. The results are summarized in Table \ref{tab:per-time-point_classification_2}).

\noindent\textbf{6.3~~Half-Cheetah kinematic modeling.} We intended to evaluate how well continuous-time models can capture physical dynamics. To perform this, we collected 25 rollouts of a pre-trained controller for the HalfCheetah-v2 gym environment \cite{brockman2016openai}, generated by the MuJoCo physics engine \cite{todorov2012mujoco}. The task is then to fit the observation space time-series in an autoregressive fashion (Fig. \ref{fig:half_cheetah}). To increase the difficulty, we overwrite $5\%$ of the actions by random actions. 
The test results are presented in Table \ref{tab:per-sequence}, and root for the superiority of the performance of LTCs compared to other models. 

\begin{figure}[t]
    \centering
    \includegraphics[width=0.40\textwidth]{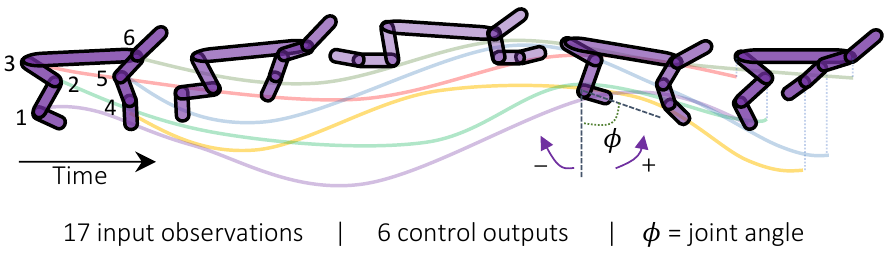}
    \caption{Half-cheetah physics simulation}
    \label{fig:half_cheetah}
\end{figure}

 \begin{table}[t]
    \centering
    \caption{Sequence modeling. Half-Cheetah dynamics n=5}
    \begin{tabular}{lc}
    \toprule
         \textbf{Algorithm} & \textbf{MSE} \\
    \hline
        LSTM & 2.500$ \pm$ 0.140  \\
        CT-RNN & 2.838$ \pm$ 0.112 \\
        Neural ODE & 3.805 $ \pm$ 0.313 \\
        CT-GRU & 3.014$ \pm$ 0.134\\
        LTC (ours) & \textbf{2.308$ \pm$ 0.015}\\
    \bottomrule
    \end{tabular}
    \label{tab:per-sequence}
\end{table}

\section{Related Works}
\noindent\textbf{Time-continuous models.} TC networks have become unprecedentedly popular. This is due to the manifestation of several benefits such as adaptive computations, better continuous time-series modeling, memory, and parameter efficiency \cite{chen2018neural}. A large number of alternative approaches have tried to improve and stabilize the adjoint method \cite{gholami2019anode}, use neural ODEs in specific contexts \cite{rubanova2019latent,lechner2019designing} and to characterize them better \cite{dupont2019augmented,durkan2019neural,jia2019neural,hanshu2020robustness,holl2020learning,quaglino2020snode}. In this work, we investigated the expressive power of neural ODEs and proposed a new ODE model to improve their expressivity and performance.

\noindent\textbf{Measures of expressivity.} A large body of modern works tried to find answers to the questions such as why deeper networks and particular architectures perform well, and where is the boundary between the approximation capability of shallow networks and deep networks? In this context, \cite{montufar2014number} and \cite{pascanu2013number} suggested to count the number of linear regions of neural networks as a measure of expressivity, \cite{eldan2016power} showed that there exists a class of radial functions that smaller networks fail to produce, and \cite{poole2016exponential} studied the exponential expressivity of neural networks by transient chaos. 

These methods are compelling; however, they are bound to particular weight configurations of a given network in order to lower-bound expressivity similar to \cite{serra2017bounding,gabrie2018entropy,hanin2018start,hanin2019complexity,lee2019towards}. \cite{raghu2017expressive} introduced an interrelated concept which quantifies the expressiveness of a given static network by trajectory length. We extended their expressivity analysis to time-continuous networks and provided lower-bound for the growth of the trajectory length, proclaiming the superior approximation capabilities of LTCs.

\section{Conclusions, Scope and Limitations}
We investigated the use of a novel class of time-continuous neural network models obtained by a combination of linear ODE neurons and special nonlinear weight configurations. We showed that they could be implemented effectively by arbitrary variable and fixed step ODE solvers, and be trained by backpropagation through time. We demonstrated their bounded and stable dynamics, superior expressivity, and superseding performance in supervised learning time-series prediction tasks, compared to standard and modern deep learning models.

\noindent\textbf{Long-term dependencies.} Similar to many variants of time-continuous models, LTCs express the vanishing gradient phenomenon \cite{pascanu2013difficulty,lechner2020learning}, when trained by gradient descent. Although the model shows promise on a variety of time-series prediction tasks, they would not be the obvious choice for learning long-term dependencies in their current format. 

\noindent\textbf{Choice of ODE solver.} Performance of time-continuous models is heavily tided to their numerical implementation approach \cite{hasani2020interpretable}. While LTCs perform well with advanced variable-step solvers and the Fused fixed-step solver introduced here, their performance is majorly influenced when off-the-shelf explicit Euler methods are used. 

\noindent\textbf{Time and Memory.} Neural ODEs are remarkably fast compared to more sophisticated models such as LTCs. Nonetheless, they lack expressivity. Our proposed model, in their current format, significantly enhances the expressive power of TC models at the expense of elevated time and memory complexity which must be investigated in the future.


\noindent\textbf{Causality.} 
Models described by time-continuous differential equation semantics inherently possess causal structures \cite{scholkopf2019causality}, especially models that are equipped with recurrent mechanisms to map past experiences to next-step predictions. Studying causality of performant recurrent models such as LTCs would be an exciting future research direction to take, as their semantics resemble \emph{dynamic causal models} \cite{friston2003dynamic} with a \emph{bilinear dynamical system} approximation \cite{penny2005bilinear}. Accordingly, a natural application domain would be the control of robots in continuous-time observation and action spaces where causal structures such as LTCs can help improve reasoning \cite{lechner2020neural}.

\section*{Acknowledgments}
R.H. and D.R. are partially supported by Boeing. R.H. and R.G. were partially supported by the Horizon-2020 ECSEL Project grant No. 783163 (iDev40). M.L. was supported in part by the Austrian Science Fund (FWF) under grant Z211-N23 (Wittgenstein Award). A.A. is supported by the National Science Foundation (NSF) Graduate Research Fellowship Program. This research work is partially drawn from the PhD dissertation of R.H.


\clearpage

\beginsupplement
\onecolumn

\noindent \text{\huge \textbf{Supplementary Materials}}
\vspace{+10mm}

\section{Proof of Theorem 1}
\begin{proof}
Assuming the neural network $f$ in Eq. \ref{eq:ltc_1}, possesses a bounded sigmoidal nonlinearity which is a monotonically increasing between $0$ and $1$. Then for each neuron $i$, we have:
 \begin{equation}
    0 < f(x_j(t), \gamma _{ij}, \mu _{ij}) < 1
 \end{equation} 
 
By replacing the upper-bound of $f$ in Eq. \ref{eq:ltc_1}, and assuming a scaling weight matrix $W_i^{M \times 1}$, for each neuron $i$ in $f$, we get: 
\begin{equation}
\label{eq:neuron1}
  \frac{d x_i}{dt} = - \Big[\frac{1}{\tau_i} + W_i \Big] x_i(t) +  W_i A_i.
\end{equation} 

The Equation simplifies to a linear ODE, of the form:
\begin{align}
\frac{dx_i}{dt} = -\underbrace{\Big[\frac{1}{\tau_i} + W_i \Big] }_\text{a} x_i
 - \underbrace{W_i A_i}_\text{b},~~\rightarrow~~\frac{dx_i}{dt} = -a x_i + b,
\label{eq:neuron2}
\end{align} 

with a solution of the form: 
\begin{equation}
\label{eq:neuron33}
 x_i(t) = k_1 e^{-at} + \frac{b}{a}.
\end{equation}
From this solution, we derive the lower bound of the system's time constant, $\tau_{sys _i}^{min}$: 
\begin{equation}
\tau_{sys _i}^{min} = \frac{1}{a} = \frac{1}{1~+~\tau_i W_i}. 
\end{equation}

By replacing the lower-bound of $f$ in Eq. \ref{eq:ltc_1}, the equation simplifies to an autonomous linear ODE as follows:

\begin{equation}
 \frac{d x_i}{dt} = - \frac{1}{\tau_i} x_i(t).
\label{eq:neuron4}
\end{equation}

which gives us the upper-bound of the system's time-constant, $\tau _{sys_i}^{max}$:

\begin{equation}
\tau _{sys_i}^{max} = \tau_i
\end{equation}
\end{proof}

\section{Proof of Theorem \ref{lemma_neuralstate}}
\begin{proof}
Let us insert $ M = max\{0, A_{i}^{max}\}$ as the neural state of neuron $i$, $x_i(t)$ into Equation \ref{eq:ltc_1}:
\begin{equation}
 \frac{d x_i}{dt} = - \Big[\frac{1}{\tau} + f(\textbf{x}_j(t), t, \theta)\Big] M + f(\textbf{x}_j(t), t, \theta) A_i.
\end{equation}
Now by expanding the brackets, we get

\begin{equation}
\label{eq:neuronmax}
 \frac{d x_i}{dt} = \underbrace{-\frac{1}{\tau}M}_\text{$\leq 0$} + \underbrace{- f(\textbf{x}_j(t), t, \theta)M + f(\textbf{x}_j(t), t, \theta) A_i}_\text{$\leq 0$}.
\end{equation}

The right-hand side of Eq. \ref{eq:neuronmax}, is negative based on the conditions on $M$, positive weights, and the fact that $f(x_j)$ is also positive, Therefore, the left-hand-side must also be negative and if we perform an approximation on the derivative term, the following holds:
\begin{equation}
 \frac{dx_i}{dt} \leq 0,~~~ 
 \frac{dx_i}{dt} \approx \frac{x_i(t+\Delta t) - x_i(t)}{\Delta t} \leq 0,
\end{equation}

By substituting $x_i(t)$ with $M$, we get: 
\begin{equation}
 \frac{x(t+\Delta t) - M}{\Delta t} \leq 0~\rightarrow~x(t+\Delta t) \leq M
\end{equation}

and therefore:
\begin{equation}
 x_i(t) \leq {max}(0, A_{i}^{max}).
\end{equation}

Now if we replace $x_{(i)}$ by $ m = min\{0, A_{i}^{min}\}$, and follow a similar methodology used for the upper bound, we can derive:
\begin{equation}
 \frac{x(t+\Delta t) - m}{\Delta t} \leq 0~\rightarrow~x(t+\Delta t) \leq m,
\end{equation}

and therefore:

\begin{equation}
 x_i(t) \geq {min}(0, A_{i}^{min}).
\end{equation}  
\end{proof}

\section{Proof of Theorem 3}
We prove that any given $n$-dimensional dynamical system for a finite simulation time can be approximated by the internal and output states of an LTC, with $n$-outputs, some hidden nodes, and a proper initial condition. We base our proof on the fundamental universal approximation theorem \cite{hornik1989multilayer} on feedforward neural networks \cite{funahashi1989approximate,cybenko1989approximation,hornik1989multilayer}, recurrent neural networks (RNN) \cite{funahashi1989approximate,schafer2006recurrent} and continuous-time RNNs \cite{funahashi1993approximation}. The fundamental difference of the proof of the universal approximation capability of LTCs compared to that of CT-RNNs lies in the distinction of the semantics of both ODE systems. LTC networks contain a nonlinear input-dependent term in their time-constant module, represented in Eq. \ref{eq:ltc_1}, which alters the entire dynamical system from that of CT-RNNs. Therefore, careful considerations have to be adjusted while taking the same approach to that of CT-RNNs for proving their universality. We first revisit preliminary statements that are used in the proof and are about basic topics on dynamical systems. 

THEOREM (The fundamental approximation theorem) \cite{funahashi1989approximate}. Let $\textbf{x}=~(x_1,...,x_n)$ be an $n$-dimensional Euclidean space $\mathbb{R}^n$. \textit{Let $f(x)$ be a sigmoidal function (a non-constant, monotonically increasing and bounded continous function in $\mathbb{R}$). Let $K$ be a compact subset of $\mathbb{R}^n$, and $f(x_1,...,x_n)$ be a continuous function on K. Then, for an arbitrary $\epsilon > 0$, there exist an integer $N$, real constants $c_i$, $\theta_i (i = 1,..., N)$ and $w_{ij}(i = 1, ..., N; j = 1, ..., n)$, such that
\begin{equation}
 \underset{x \in K}{max} |g(x_1,..., x_n) - \sum^{N}_{i=1} c_i f(\sum^{n}_{j=1} w_{ij} x_j - \theta_i)| < \epsilon
\end{equation}
holds.} 

This theorem illustrates that three-layer feedforward neural networks (Input-hidden layer-output), can approximate any continuous mapping $g: \mathbb{R}^n \rightarrow \mathbb{R}^m$ on a compact set.

THEOREM (Approximation of dynamical systems by continuous time recurrent neural networks) \cite{funahashi1993approximation}.  \textit{Let $D\subset\mathbb{R}^n$ and $F:D\rightarrow\mathbb{R}^n$ be an autonomous ordinary differential equation and $C^1$-mapping, and let $\dot{\textbf{x}} = F(\textbf{x})$ determine a dynamical system on $D$. Let $K$ denote a compact subset of $D$ and we consider the trajectories of the system on the interval $I = [0,~T]$. Then, for an arbitrary positive $\epsilon$, there exist an integer $N$ and a recurrent neural network with $N$ hidden units, $n$ output units, and an output internal state $\textbf{u}(t) = (U_1(t), ..., U_n(t))$, expressed as:
\begin{equation}
 \frac{du_i(t)}{dt} = - \frac{u_i(t)}{\tau _i} + \sum^{m}_{j=1} w_{ij} f(u_j(t)) + I_i(t),
\end{equation}
where $\tau_i$ is the time-constant, $w_{ij}$ are the weights, $I_i(t)$ is the input, and $f$ is a $C^1$-sigmoid function ($f(x) = 1/(1 + exp(-x))$, such that for any trajectory $\{x(t); t \in I\}$ of the system with initial value $x(0) \in K$, and a proper initial condition of the network the statement below holds:
 \begin{center}
  $\underset{t \in I}{max} |\textbf{x}(t) - \textbf{u}(t)|< \epsilon$.
 \end{center}}
  
The theorem was proved for the case where the time-constants, $\tau$, were kept constant for all hidden states, and the RNN was without inputs ($I_i(t) = 0$) \cite{funahashi1993approximation}.

We now restate the necessary concepts from dynamical systems to be used in the proof. Where necessary, we adopt modifications and extensions to the Lemmas, for proving Theorem 1. 

\textbf{Lipschitz.} The mapping $F: S\rightarrow\mathbb{R}^n$, where S is an open subset of $\mathbb{R}^n$, is called Lipschitz on $S$ if there exist a constant $L$ (Lipschitz constant), such that: 

\begin{equation}
 |F(x) - F(y)| \leq L |x-y|,~~~ \forall x, y \in S.
\end{equation}

\textbf{Locally Lipschitz.} If every point of $S$ has neighborhood $S_0$ in $S$, such that the restriction $F~|~S_0$ is Lipschitz, then $F$ is locally Lipschitz.

\begin{lemma}
\label{lem1}
 Let a mapping $F: S\rightarrow\mathbb{R}^n$ be $C^1$. Then F is locally Lipschitz. Also, if $D \subset S$ is compact, then the restriction $F~|~D$ is Lipschitz. (Proof in \cite{morris1973differential}, chapter 8, section 3).
\end{lemma}

\begin{lemma}
\label{lem2}
 Let $F: S \rightarrow \mathbb{R}^n$ be a $C^1$-mapping and $x_0 \in S$. There exists a positive $a$ and a unique solution $x: (-a, a) \rightarrow S$ of the differential equation
 \begin{equation}
  \dot x = F(x),
 \end{equation}
 which satisfies the initial condition $x(0) = x_0$. (Proof in \cite{morris1973differential}, chapter 8, section 2, Theorem 1.)
\end{lemma}

\begin{lemma}
\label{lem3}
 Let $S$ be an open subset of $\mathbb{R}^n$ and $F: S \rightarrow \mathbb{R}^n$ be a $C^1$-mapping. On a maximal interval $J = (\alpha, \beta) \subset \mathbb{R}$, let x(t) be a solution. Then for any compact subset $D \subset S$, there exists some $t \in (\alpha, \beta)$, for which $x(t) \notin D$. (Proof in \cite{morris1973differential}, Chapter 8, section 5, Theorem).
\end{lemma}

\begin{lemma}
\label{lem4}
 For an $F: \mathbb{R}^n \rightarrow \mathbb{R}^n$ which is a bound $C^1$-mapping, the differential equation
 \begin{equation}
  \dot x = -\frac{x}{\tau} + F(x), 
 \end{equation}
  where $\tau > 0 $ has a unique solution on $[0, \infty)$. (Proof in \cite{funahashi1993approximation}, Section 4, Lemma 4).
\end{lemma}

\begin{lemma}
\label{lem5}
 For an $F: \mathbb{R}^n \rightarrow \mathbb{R^+}^n$ which is a bounded $C^1$-mapping, the differential equation
 \begin{equation}
 \label{eq7-a}
  \dot x = - (1/\tau +F(x))x + A F(x),
 \end{equation}
 in which $\tau$ is a positive constant, and $A$ is constant coefficients bound to a range $[-\alpha, \beta]$ for $ 0 <\alpha < +\infty$, and $0 \leq \beta < +\infty$,
 has a unique solution on $[0, \infty)$.
\end{lemma}
\begin{proof}
 Based on the assumptions, we can take a positive $M$, such that 
 \begin{equation}
  0 \leq F_i(x) \leq M (\forall i = 1,...,n)
 \end{equation}
 by looking at the solutions of the following differential equation: 
 \begin{equation}
  \dot x = - (1/\tau +M)x + A M,
 \end{equation} 
we can show that
\begin{equation}
 min\{|x_i(0)|, \frac{\tau(A M)}{1+\tau M}\} \leq x_i(t) \leq max\{|x_i(0)|, \frac{\tau(A M)}{1+\tau M}\}, 
\end{equation}
if we set the output of the max to $C_{max_i}$ and the output of the min to $C_{min_i}$ and also set $C_1 = min\{C_{min_i}\}$ and $C_2 = max\{C_{max_i}\}$, then the solution $x(t)$ satisfies
\begin{equation}
 \sqrt{n}C_1 \leq x(t) \leq \sqrt{n}C_2.
\end{equation}
Based on Lemma \ref{lem2} and Lemma \ref{lem3} a unique solution exists on the interval $[0,+\infty)$. 
\end{proof}

Lemma \ref{lem5} demonstrates that an LTC network defined by Eq. \ref{eq7-a}, has a unique solution on $[0,\infty)$, since the output function is bounded and is a $C^1$ mapping. 

\begin{lemma}
 \label{lem6}
Let two continuous mapping $F,\tilde F:S \rightarrow \mathbb{R}^n$ be Lipschitz, and $L$ be a Lipschitz constant of $F$. if $\forall x\in S$, 
\begin{equation}
 |F(\textbf{x}) - \tilde F(\textbf{x})| < \epsilon,
\end{equation}
holds, if $\textbf{x}(t)$ and $\textbf{y}(t)$ are solutions to
\begin{equation}
 \dot{\textbf{x}} = F(\textbf{x}),
\end{equation}
\begin{equation}
 \dot{\textbf{y}} = \tilde F(\textbf{x}),
\end{equation}
on some interval $J$, such that $x(t_0) = y(t_0)$, then
\begin{equation}
 |\textbf{x}(t) - \textbf{y}(t)| \leq \frac{\epsilon}{L}(e^{L|t-t_0|} - 1).
\end{equation}
(Proof in \cite{morris1973differential}, chapter 15, section 1, Theorem 3).
\end{lemma}

\subsection{Proof of the Theorem:}
\begin{proof}
Using the above definitions and lemmas, we prove that LTCs are universal approximators.

Part 1. We choose an $\eta$ which is in range $(0, min\{\epsilon, \lambda\})$, for $\epsilon > 0$, and $\lambda$ the distance between $\tilde D$ and boundary $\delta S$ of $S$. $D_{\eta}$ is set:
\begin{equation}
 D_{\eta} = \{ \textbf{x} \in \mathbb{R}^n; \exists z \in \tilde D, |\textbf{x}-\textbf{z}| \leq \eta \}.
\end{equation}
$D_{\eta}$ stands for a compact subset of $S$, because $\tilde D$ is compact. Thus, $F$ is Lipschitz on $D_{\eta}$ by Lemma \ref{lem1}. Let $L_F$ be the Lipschitz constant of $F | D_{\eta}$, then, we can choose an $\epsilon _l > 0$, such that
\begin{equation}
 \epsilon _l < \frac{\eta L_F}{2(e^{L_F T-1)}}.
\end{equation}

Based on the universal approximation theorem, there is an integer $N$, and an $n \times N$ matrix $A$, and an $N \times n$ matrix $C$ and an $N$-dimensional vector $\mu$ such that
\begin{equation}
\label{eq26}
 max |F(\textbf{x}) - A f(\gamma \textbf{x} + \mu)| < \frac{\epsilon_l}{2}.
\end{equation}

We define a $C^1$-mapping $\tilde F: \mathbb{R}^n \rightarrow \mathbb{R}^n$ as: 
\begin{equation}
\label{eq27}
 \tilde F(\textbf{x}) = - (1/\tau +W_l f(\gamma \textbf{x}+\mu))\textbf{x} + W_l f(\gamma \textbf{x} + \mu) A,
\end{equation}
with parameters matching that of Eq. \ref{eq:ltc_1} with $W_l =W$. 

We set system's time-constant, $\tau_{sys}$ as:

\begin{equation}
 \tau_{sys} = \frac{1}{\tau/1 +\tau W_l f(\gamma x+\mu)}. 
\end{equation}
We chose a large $\tau_{sys}$, conditioned with the following:

\begin{align}
&(a)~~\forall x \in D_\eta;~~| \frac{x}{\tau _{sys}}| < \frac{\epsilon _l}{2}\\
&(b)~~| \frac{\mu}{\tau _{sys}}|< \frac{\eta L_{\tilde G}}{2(e^{L_{\tilde G} T}- 1)}~\text{and}~| \frac{1}{\tau _{sys}}| < \frac{L_{\tilde G}}{2},
\end{align}

where $L_{\tilde G}/2$ is a lipschitz constant for the mapping $W_l f: \mathbb{R}^{n+N} \rightarrow \mathbb{R}^{n+N}$ which we will determine later. To satisfy conditions (a) and (b), $\tau W_l << 1$ should hold true.

Then by Eq. \ref{eq26} and \ref{eq27}, we can prove:
\begin{equation}
 \underset{x \in D_\eta}{max} | F(\textbf{x}) - \tilde F(\textbf{x})| < \epsilon _l
\end{equation}

Let's set $\textbf{x}(t)$ and $\tilde{\textbf{x}}t)$ with initial state $x(0) = \tilde x(0) = x_0 \in D$, as the solutions of equations below:
\begin{equation}
 \dot{\textbf{x}} = F(\textbf{x}),
\end{equation}
\begin{equation}
 \dot{\tilde{\textbf{x}}} = \tilde F(\textbf{x}).
\end{equation}

Based on Lemma \ref{lem6} for any $t \in I$,

\begin{align}
| \textbf{x}(t) - \tilde{\textbf{x}}(t) | &\leq \frac{\epsilon _l}{L_F}(e^{L_F t} -1) \\
&\leq \frac{\epsilon _l}{L_F}(e^{L_F T} -1).
\end{align}

Thus, based on the conditions on $\epsilon$,

\begin{equation}
\label{eq36}
 \underset{t \in I}{max} | \textbf{x}(t) - \tilde{\textbf{x}}(t)|< \frac{\eta}{2}.
\end{equation}

Part 2. Let's Consider the following dynamical system defined by $\tilde F$ in Part 1:

\begin{equation}
\label{eq37}
 \dot{\tilde{\textbf{x}}} = -\frac{1}{\tau_{sys}} \tilde{\textbf{x}} + W_l f(\gamma \tilde{\textbf{x}} + \mu) A.
\end{equation}

Suppose we set $\tilde{\textbf{y}} = \gamma \tilde{\textbf{x}} + \mu$; then:
\begin{equation}
 \dot{\tilde{\textbf{y}}} = \gamma  \dot{\tilde{\textbf{x}}} = -\frac{1}{\tau_{sys}} \tilde{\textbf{y}} + E f(\tilde{\textbf{y}}) + \frac{\mu}{\tau_{sys}},
\end{equation}
where $E = \gamma W_l A$, an $N \times N$ matrix. We define

\begin{equation}
 \tilde{\textbf{z}} = (\tilde x_1, ..., \tilde x_n, \tilde y_1,...,\tilde y_n), 
\end{equation}
and we set a mapping $\tilde G: \mathbb{R}^{n+N} \rightarrow \mathbb{R}^{n+N}$ as:
\begin{equation}
\label{eq40}
  \tilde G(\tilde{\textbf{z}}) = -\frac{1}{\tau_{sys}} \tilde{\textbf{z}} + W f(\tilde{\textbf{z}}) + \frac{\mu _1}{\tau_{sys}},
\end{equation}

where;

\begin{align}
 &W^{(n+N)\times(n+N)} = \left(\begin{array}{cc} 0 & A\\ 0 & E \end{array}\right), \\
 &\mu _1^{n+N} = \left(\begin{array}{c} 0 \\ {\mu} \end{array}\right). 
\end{align}

Now using Lemma \ref{lem2}, we can show that solutions of the following dynamical system:
\begin{equation}
 \dot{\tilde{\textbf{z}}} = \tilde G(\tilde{\textbf{z}}),~~~~\tilde y(0) = \gamma \tilde x(0) + \mu, 
\end{equation}
are equivalent to the solutions of the Eq. \ref{eq37}. 

Let's define a new dynamical system $G: \mathbb{R}^{n+N} \rightarrow \mathbb{R}^{n+N}$ as follows:
\begin{equation}
\label{eq44}
 G(\textbf{z}) = -\frac{1}{\tau_{sys}} \textbf{z} + W f(\textbf{z}),
\end{equation}
where $\textbf{z} = (x_1, ..., x_n, y_1,...,y_n)$. Then the dynamical system below
\begin{equation}
\label{eqzzz}
 \dot{\textbf{z}} = -\frac{1}{\tau_{sys}} \textbf{z} + W f(\textbf{z}),
\end{equation}
can be realized by an LTC, if we set $\textbf{h}(t) = (h_1(t),...,h_N(t))$ as the hidden states, and $\textbf{u}(t) = (U_1(t),...,U_n(t))$ as the output states of the system. Since $\tilde G$ and $G$ are both $C^1$-mapping and $f^{\prime}(\textbf{x})$ is bound, therefore, the mapping $\tilde{\textbf{z}} \mapsto W f(\tilde{\textbf{z}})$ is Lipschitz on $\mathbb{R}^{n+N}$, with a Lipschitz constant $L_{\tilde G}/2$. As $L_{\tilde G}/2$ is lipschitz constant for $- \tilde z/\tau_{sys}$ by condition (b) on $\tau_{sys}$, $L_{\tilde G}$ is a Lipschitz constant of $\tilde G$. 

From Eq. \ref{eq40}, Eq. \ref{eq44}, and condition (b) of $\tau_{sys}$, we can derive the following:
\begin{equation}
 | \tilde G(\textbf{z}) - G(\textbf{z})| = | \frac{\mu}{\tau_{sys}}| < \frac{\eta L_{\tilde G}}{2(e^{L_{\tilde G}T}-1)}.
\end{equation}
Accordingly, we can set $\tilde{\textbf{z}}(t)$ and $\textbf{z}(t)$, solutions of the dynamical systems:
\begin{equation}
 \dot{\tilde{\textbf{z}}} = \tilde G(\textbf{z}),~~~ \begin{cases} \tilde x(0) = x_0 \in D \\ \tilde y(0) =  \gamma x_0 +\mu \end{cases}
\end{equation} 
\begin{equation}
 \dot{\textbf{z}} = G(\textbf{z}),~~~ \begin{cases} u(0) = x_0 \in D \\ \tilde h(0) = \gamma x_0 +\mu \end{cases}
\end{equation} 
By Lemma \ref{lem6}, we achieve
\begin{equation}
 \underset{t \in I}{max} | \tilde{\textbf{z}}(t) - \textbf{z}(t)|< \frac{\eta}{2},
\end{equation}
and therefore we have:
\begin{equation}
\label{eq50}
 \underset{t \in I}{max} | \tilde{\textbf{x}}(t) - \textbf{u}(t)|< \frac{\eta}{2},
\end{equation}

Part3. Now by using Eq. \ref{eq36} and Eq. \ref{eq50}, for a positive $\epsilon$, we can design an LTC with internal dynamical state $\textbf{z}(t)$, with $\tau_{sys}$ and $W$. For x(t) satisfying $\dot{\textbf{x}} = F(\textbf{x})$, if we initialize the network by $u(0) = x(0)$ and $h(0) = \gamma x(0) +\mu$, we obtain:
\begin{equation}
 \underset{t \in I}{max} | \textbf{x}(t) - \textbf{u}(t)|< \frac{\eta}{2}+\frac{\eta}{2}= \eta < \epsilon.
\end{equation}
\end{proof}

REMARKS. LTCs allow the elements of the hidden layer to have recurrent connections to each other. However, it assumes a feed-forward connection stream from hidden nodes to output units. We assumed no inputs to the system and principally showed that the hidden nodes' together with output units, could approximate any finite trajectory of an autonomous dynamical system. 

\section{Proof of Theorem 4}

In this section, we describe our mathematical notions and revisit concepts that are required to state the proof. The main statements of our theoretical results about the expressive power of time-continuous neural networks are chiefly built over the expressivity measure, \emph{trajectory length}, introduced for static deep neural networks in \cite{raghu2017expressive}. It is therefore intuitive to follow similar steps with careful considerations, due to the continuous nature of the models.

\subsection{Notations}

\textit{Neural network architecture --} We determine a neural network architecture by $f_{n,k}(x(t),I(t),\theta)d$, with n layers (depth),width k and total number of neurons, $N= n \times k$.

\textit{Neural state, x(t) --} For a layer d of a network $f$, $x^{(d)}(t)$ represent the neural state of the layer and is a matrix of the size $k \times m$, with m being the size of the input time series. 

\textit{Inputs, I(t) --} is a 2-dimensional matrix containing a 2-D trajectory for $t \in [0,t_{max}]$.

\textit{Network parameters, $\theta$ --} include weights matrices for each layer d of the form $W^{(d)} \sim \mathcal{N}(0,\sigma^{2}_{w}/k)$ and bias vectors as $b^{(d)} \sim \mathcal{N}(0,\sigma^{2}_{b})$. For CT-RNNs the vector parameter $\tau^{(d)}$ is also sampled from $\sim \mathcal{N}(0,\sigma^{2}_{b})$

\textit{Perpendicular and parallel components --} For given vectors $x$ and $y$ we can decompose each vector in respect to one another as $y = y_{\parallel} + y_{\bot}$. That is, $y_{\parallel}$ stands for component of $y$ parallel to $x$ and $y_{\bot}$ is the perpendicular component in respect to $x$.

\textit{Weight matrix decomposition --} \cite{raghu2017expressive} showed that for given non-zero vectors $x$ and $y$, and a full rank matrix $W$, one can write a matrix decomposition for $W$ in respect to $x$ and $y$ as follows: $W = \tensor[^\parallel]{W}{_\parallel} + \tensor[^\parallel]{W}{_\bot} + \tensor[^\bot]{W}{_\parallel} + \tensor[^\bot]{W}{_\bot}$, such that, $\tensor[^\parallel]{W}{_\bot} x = 0$, $\tensor[^\bot]{W}{_\bot} x = 0$, $ y^T \tensor[^\bot]{W}{_\parallel} = 0$ and $ y^T \tensor[^\bot]{W}{_\bot} = 0$. In this notation, the decomposition superscript on left is in respect to $y$ and the subscript on right is in respect to $x$. It has also been observed that $W_{\bot}$ in respect to x can be obtained by: $W_{\bot} = W - W_{\parallel}$ \cite{raghu2017expressive}.

\begin{lemma}
\label{lem:independence_of_projections}
\emph{Independence of Projections \cite{raghu2017expressive}.} Given a matrix W with iid entries drawn form $\mathcal{N}(0,\sigma^{2})$, then its decomposition matrices $W_{\bot}$ and $W_{\parallel}$ in respect to x, are independent random variables. 
\end{lemma}
Proof in \cite{raghu2017expressive}, Appendix, Lemma 2.

\begin{lemma}
\label{lem:norm_of_gaussian_vecor}
\emph{Norm of Gaussian Vector \cite{raghu2017expressive}.} The norm of a Gaussian vector $X \in \mathbb{R}^k$, with its entries sampled iid $\sim \mathcal{N}(0,\sigma^2)$ is given by:
\begin{equation}
    \mathbb{E}[\norm{X}] = \sigma \sqrt{2} \frac{\Gamma((k+1)/2)}{\Gamma(k/2)}.
\end{equation}
\end{lemma}
Proof in \cite{raghu2017expressive}, Appendix, Lemma 3.

\begin{lemma}
\label{lem:norm_of_projections}
\emph{Norm of Projections \cite{raghu2017expressive}.} for a $W^{k \times k}$ with conditions of Lemma \ref{lem:norm_of_gaussian_vecor}, and two vectors, $x$ and $y$, then the following holds for $x_{\bot}$ being a non-zero vector, perpendicular to $x$:
\begin{equation}
    \mathbb{E}[\norm{\tensor[^\bot]{W}{_\bot} x_{\bot}}] = \norm{x_{\bot}} \sigma \sqrt{2} \frac{\Gamma((k)/2)}{\Gamma((k-1)/2)} \geq \norm{x_{\bot}} \sigma \sqrt{2} (\frac{k}{2}-\frac{3}{4})^{1/2}.
\end{equation}
It has also been shown in \cite{raghu2017expressive}: "that if $1_{\mathcal{A}}$ is an identity matrix with non-zero diagonal entry $i$ iff $i \in \mathcal{A} \subset [k]$ and $|\mathcal{A}|>2$, then:

\begin{equation}
    \mathbb{E}[\norm{1_{\mathcal{A}}\tensor[^\bot]{W}{_\bot} x_{\bot}}] = \norm{x_{\bot}} \sigma \sqrt{2} \frac{\Gamma(|\mathcal{A}|/2)}{\Gamma((|\mathcal{A}|-1)/2)} \geq \norm{x_{\bot}} \sigma \sqrt{2} (\frac{|\mathcal{A}|}{2}-\frac{3}{4})^{1/2}.~\emph{"}
\end{equation}
\end{lemma}
Proof in \cite{raghu2017expressive}, Appendix, Lemma 4.

\begin{lemma}
\label{lem:norm_and_translation}
\emph{Norm and Translation \cite{raghu2017expressive}.} For $X$ being a zero-mean multivariate Gaussian and having a diagonal covariance matrix, and $\mu$ a vector of constants, we have:
\begin{equation}
    \mathbb{E}[\norm{X - \mu}] \geq \mathbb{E}[\norm{X}].
\end{equation}
\end{lemma}
Proof in \cite{raghu2017expressive}, Appendix, Lemma 5.

\subsection{Beginning of the proof of Theorem 4}
We first establish the lower bound for Neural ODEs and then extend the results to that of CT-RNNs.
\begin{proof}
For a successive layer $d+1$ of a Neural ODE the gradient between the states at $t+\delta t$ and $t$, $x^{d+1}(t+\delta t)$ and $x^{d+1}(t)$ is determined by:

\begin{equation}
 \frac{dx}{dt}^{(d+1)} = f(h^{(d)}),~~~~~~~h^{(d)} = W^{(d)}x^{(d)}+b^{(d)}.
\end{equation}

Accordingly, for the latent representation (the first two principle components of the hidden state $x^{(d+1)}$), which is denoted by $z^{(d+1)}(t)$, this gradient can be determined by:

\begin{equation}
 \frac{dz}{dt}^{(d+1)} = f(h^{(d)}),~~~~~~~h^{(d)} = W^{(d)}z^{(d)}+b^{(d)}
\end{equation}

Let us continue with the zero bias case and discuss the non-zero bias case later. 

We decompose $W^{(d)}$ in respect to the $z^{(d)}$, as $W^{(d)} = W^{(d)}_{\parallel} + W^{(d)}_{\bot}$. For this decomposition, the hidden state $h^{(d+1)} = W^{(d)}_{\parallel}z^{(d)}$ as the vertical components maps $z^{(d)}$ to zero. 

We determine the set of indices for which the gradient state is not saturated as if $f$ is defined by \url{Hard-tanh} activations:
\begin{equation}
    \mathcal{A}_{W^{(d)}_{\parallel}} = \{i : i \in [k], |h_{i}^{(d+1)}| < 1 \}
\end{equation}

As the decomposition components of $W^{(d)}$ are independent random variables, based on Lemma \ref{lem:norm_of_projections}, we can build the expectation of the gradient state as follows:
\begin{equation}
\label{eq:similar_1}
    \mathbb{E}_{W^{(d)}}\Bigg[\norm{\frac{dz}{dt}^{(d+1)}}\Bigg] = \mathbb{E}_{W^{(d)}_{\parallel}}\mathbb{E}_{W^{(d)}_{\bot}}\Big[\norm{f(W^{(d)}z^{(d)})}\Big].
\end{equation}

Now, if we condition on $W^{(d)}_{\parallel}$, we can replace the right-hand-side norm with the sum over the non-saturated indices, $\mathcal{A}_{W^{(d)}_{\parallel}}$ as follows:
\begin{equation}
\label{eq:expected_global}
    \mathbb{E}_{W^{(d)}}\Bigg[\norm{\frac{dz}{dt}^{(d+1)}}\Bigg] = \mathbb{E}_{W^{(d)}_{\parallel}}\mathbb{E}_{W^{(d)}_{\bot}}\Bigg[\Bigg(\sum_{i \in \mathcal{A}_{W^{(d)}_{\parallel}}} \big( (W^{(d)}_{\bot})_{i}\ z^{(d)} + (W^{(d)}_{\parallel})_{i}\ z^{(d)} \big)^{2}\Bigg)^{1/2}\Bigg].
\end{equation}

We need to derive a recurrence for the Eq. \ref{eq:expected_global}. To do this, we start a decomposition of the gradient state in respect to $z^{(d)}$ as $\frac{dz}{dt}^{(d)} = \frac{dz}{dt}^{(d)}_{\parallel} + \frac{dz}{dt}^{(d)}_{\bot}$. 

Now, let $\frac{\tilde{dz}}{dt}^{(d+1)} = 1_{\mathcal{A}_{W^{(d)}_{\parallel}}} h^{(d+1)}$, be the latent gradient vector of all unsaturated units, and zeroed saturated units. Also we decompose the column space of the weight matrix in respect to $\tilde{z}^{(d+1)}$ as: $\tensor[]{W}{^{(d)}} = \tensor[^\bot]{W}{^{(d)}} + \tensor[^\parallel]{W}{^{(d)}}$.

Then by definition, we have the following expressions:
\begin{equation}
\label{eq:def_perpendi_dzdt}
\frac{dz}{dt}^{(d+1)}_{\bot} = W^{(d)}z^{(d)} 1_{\mathcal{A}} - \langle W^{(d)}z^{(d)} 1_{\mathcal{A}}, \hat z^{(d+1)} \rangle \hat z^{(d+1)},~~~~~ \hat . = unit~vector
\end{equation}

\begin{equation}
\label{eq:def_perpendi_Wz}
 \tensor[^\bot]{W}{^{(d)}} z^{(d)} = \tensor[]{W}{^{(d)}} z^{(d)} -  \langle \tensor[]{W}{^{(d)}} z^{(d)}, \hat{\tilde{z}}^{(d+1)}\rangle \hat{\tilde{z}}^{(d+1)}
\end{equation}

Looking at Eq. \ref{eq:def_perpendi_dzdt} and Eq. \ref{eq:def_perpendi_Wz}, and based on the definitions provided, their right-hand-side are equal to each other for any $i \in \mathcal{A}$. Therefore, their left-hand-sides are equivalent as well. More precisely:

\begin{equation}
\label{eq:the_statement_to_start_recurrence}
    \frac{dz}{dt}^{(d+1)}_{\bot} . 1_{\mathcal{A}} =  \tensor[^\bot]{W}{^{(d)}} z^{(d)} . 1_{\mathcal{A}}. 
\end{equation}

The statement in Eq. \ref{eq:the_statement_to_start_recurrence} allows us to determine the following inequality, which builds up the first steps for the recurrence:

\begin{equation}
\label{eq:first_statement_to_recurrence}
    \norm{\frac{dz}{dt}^{(d+1)}_{\bot}} \geq \norm{ \frac{dz}{dt}^{(d+1)}_{\bot} . 1_{\mathcal{A}}}
\end{equation}

Now let us return to Eq. \ref{eq:expected_global}, and plug in the following decompositions:
\begin{equation}
    \frac{dz}{dt}^{(d)} = \frac{dz}{dt}^{(d)}_{\bot} + \frac{dz}{dt}^{(d)}_{\parallel} 
\end{equation}
\begin{equation}
    \tensor[]{W}{_{\bot}}^{(d)} = \tensor[^\parallel]{W}{_{\bot}}^{(d)} + \tensor[^\bot]{W}{_{\bot}}^{(d)}~~~~~~~~~~~~\tensor[]{W}{_{\parallel}}^{(d)} = \tensor[^\parallel]{W}{_{\parallel}}^{(d)} + \tensor[^\bot]{W}{_{\parallel}}^{(d)},
\end{equation}

we have:

\begin{gather}
\label{eq:expected_global_recurrence}
    \mathbb{E}_{W^{(d)}}\Bigg[\norm{\frac{dz}{dt}^{(d+1)}}\Bigg] =\\ \mathbb{E}_{W^{(d)}_{\parallel}}\mathbb{E}_{W^{(d)}_{\bot}}\Bigg[\Bigg(\sum_{i \in \mathcal{A}_{W^{(d)}_{\parallel}}} \big( (\tensor[^\parallel]{W}{_{\bot}}^{(d)} + \tensor[^\bot]{W}{_{\bot}}^{(d)})_{i}\ z^{(d)}_{\bot} + (\tensor[^\parallel]{W}{_{\parallel}}^{(d)} + \tensor[^\bot]{W}{_{\parallel}}^{(d)})_{i}\ z^{(d)}_{\parallel} \big)^{2}\Bigg)^{1/2}\Bigg]
\end{gather}

As stated in Theorem \ref{theorem:neural_ode}, we conditioned the input on its perpendicular components. Therefore, we write the recurrence of the states also for their perpendicular components by dropping the parallel components, $\tensor[^\parallel]{W}{_{\bot}}^{(d)}$ and $\tensor[^\parallel]{W}{_{\parallel}}^{(d)}$, and using Eq. \ref{eq:first_statement_to_recurrence} as follows:

\begin{gather}
\label{eq:second_statement_to_recurrence}
    \mathbb{E}_{W^{(d)}}\Bigg[\norm{\frac{dz}{dt}^{(d+1)}_{\bot}} \Bigg] \geq \mathbb{E}_{W^{(d)}_{\parallel}}\mathbb{E}_{W^{(d)}_{\bot}}\Bigg[\Bigg(\sum_{i \in \mathcal{A}_{W^{(d)}_{\parallel}}} \big( (\tensor[^\bot]{W}{_{\bot}}^{(d)})_{i}\ z^{(d)}_{\bot} + (\tensor[^\bot]{W}{_{\parallel}}^{(d)})_{i}\ z^{(d)}_{\parallel} \big)^{2}\Bigg)^{1/2}\Bigg]
\end{gather}

The term $\tensor[^\bot]{W}{_{\parallel}}^{(d)} z^{(d)}_{\parallel}$ is constant, as the inner expectation is conditioned on $W^{(d)}_{\parallel}$. Now by using Lemma \ref{lem:norm_and_translation}, we can wirte:

\begin{gather}
\label{eq:third_statement_to_recurrence}
    \mathbb{E}_{W^{(d)}_{\bot}}\Bigg[\Bigg(\sum_{i \in \mathcal{A}_{W^{(d)}_{\parallel}}} \big( (\tensor[^\bot]{W}{_{\bot}}^{(d)})_{i}\ z^{(d)}_{\bot} + (\tensor[^\bot]{W}{_{\parallel}}^{(d)})_{i}\ z^{(d)}_{\parallel} \big)^{2}\Bigg)^{1/2}\Bigg] \geq \\ 
    \mathbb{E}_{W^{(d)}_{\bot}}\Bigg[\Bigg(\sum_{i \in \mathcal{A}_{W^{(d)}_{\parallel}}} \big( (\tensor[^\bot]{W}{_{\bot}}^{(d)})_{i}\ z^{(d)}_{\bot}\big)^{2}\Bigg)^{1/2}\Bigg]
\end{gather}

By applying Lemma \ref{lem:norm_of_projections} we get:

\begin{gather}
\label{eq:forth_statement_to_recurrence}
    \mathbb{E}_{W^{(d)}_{\bot}}\Bigg[\Bigg(\sum_{i \in \mathcal{A}_{W^{(d)}_{\parallel}}} \big( (\tensor[^\bot]{W}{_{\bot}}^{(d)})_{i}\ z^{(d)}_{\bot}\big)^{2}\Bigg)^{1/2}\Bigg]
    \geq \frac{\sigma_{w}}{\sqrt{k}}\sqrt{2}\frac{\sqrt{2|\mathcal{A}_{W^{(d)}_{\parallel}}|-3}}{2} \mathbb{E}\Big[\norm{z^{(d)}_{\bot}}\Big].
\end{gather}

As we selected \url{Hard-tanh} activation functions with $p = \mathbb{P}(|h^{(d+1)}_{i} |< 1)$, and the condition $|\mathcal{A}_{W^{(d)}_{\parallel}}| \geq 2$ we have $\sqrt{2}\frac{\sqrt{2|\mathcal{A}_{W^{(d)}_{\parallel}}|-3}}{2} \geq \frac{1}{\sqrt{2}}\sqrt{|\mathcal{A}_{W^{(d)}_{\parallel}}|}$, and therefore we get: 

\begin{gather}
\label{eq:fifth_statement_to_recurrence}
    \mathbb{E}_{W^{(d)}}\Bigg[\norm{\frac{dz}{dt}^{(d+1)}_{\bot}} \Bigg] 
    \geq
    \frac{1}{\sqrt{2}} \Bigg(\sum_{j=2}^{k} \begin{pmatrix} k \\ j \end{pmatrix} p^j {(1-p)}^{k-j} \frac{\sigma_{w}}{\sqrt{k}}\sqrt{j} \Bigg) \mathbb{E}\Big[\norm{z^{(d)}_{\bot}}\Big]
\end{gather}

Keep in mind that we are referring to $|\mathcal{A}_{W^{(d)}_{\parallel}}|$ as $j$. Now we need to bound the $\sqrt{j}$ term, by considering the binomial distribution represented by the sum. Consequently, we can rewrite the sum in Eq. \ref{eq:fifth_statement_to_recurrence} as follows:

\begin{gather*}
    \sum_{j=2}^{k} \begin{pmatrix} k \\ j \end{pmatrix} p^j {(1-p)}^{k-j} \frac{\sigma_{w}}{\sqrt{k}}\sqrt{j}  = - \begin{pmatrix} k \\ 1 \end{pmatrix} p^j {(1-p)}^{k-1} \frac{\sigma_{w}}{\sqrt{k}} + \sum_{j=2}^{k} \begin{pmatrix} k \\ j \end{pmatrix} p^j {(1-p)}^{k-j} \frac{\sigma_{w}}{\sqrt{k}}\sqrt{j} \\
    = -\sigma_{w} \sqrt{k}p(1-p)^{k-1} + k p \frac{\sigma_{w}}{\sqrt{k}} \underbrace{\sum_{j=2}^{k} \frac{1}{\sqrt{j}}\begin{pmatrix} k-1 \\ j-1 \end{pmatrix} p^{j-1} {(1-p)}^{k-j}}_\text{XT}
\end{gather*}

and by utilizing Jensen's inequality with $1/\sqrt{x}$, we can simplify XT as follows as it is the expectation of the binomial distribution $(k-1,p)$ \cite{raghu2017expressive}:

\begin{gather*}
    \sum_{j=2}^{k} \frac{1}{\sqrt{j}}\begin{pmatrix} k-1 \\ j-1 \end{pmatrix} p^{j-1} {(1-p)}^{k-j} \geq \frac{1}{\sqrt{\sum_{j=2}^{k} j \begin{pmatrix} k-1 \\ j-1 \end{pmatrix} p^{j-1} {(1-p)}^{k-j}}} = \frac{1}{\sqrt{(k-1)p+1}}
\end{gather*}

and therefore:

\begin{gather}
\label{eq:pp}
    \mathbb{E}_{W^{(d)}}\Bigg[\norm{\frac{dz}{dt}^{(d+1)}_{\bot}} \Bigg] 
    \geq
    \frac{1}{\sqrt{2}} \Bigg( -\sigma_{w} \sqrt{k}p(1-p)^{k-1} + \sigma_{w} \frac{\sqrt{k} p}{{\sqrt{(k-1)p+1}}} \Bigg) \mathbb{E}\Big[\norm{z^{(d)}_{\bot}}\Big]
\end{gather}

Now we need to find a range for $p$. \cite{raghu2017expressive} showed that for \url{Hard-tanh} activations, given the fact that $h^{(d+1)}_{i}$ is a random variable with variance less than $\sigma_w$, for an input argument $|A| \sim \mathcal{N}(0,\sigma^{2}_{w})$, we can lower bound $p = \mathbb{P}(|h^{(d+1)}_{i} |< 1)$, as follows:

\begin{equation}
    p = \mathbb{P}(|h^{(d+1)}_{i} |< 1) \geq \mathbb{P}(|A|< 1) \geq \frac{1}{\sqrt{2\pi}\sigma_{w}},~~~~ \forall~\sigma_{w} \geq 1,
\end{equation}

and find an upper bound equal to $\frac{1}{\sigma_w}$ \cite{raghu2017expressive}. Therefore the equation becomes:

\begin{gather}
\label{eq:sixth_statement_to_recurrence}
    \mathbb{E}_{W^{(d)}}\Bigg[\norm{\frac{dz}{dt}^{(d+1)}_{\bot}} \Bigg] 
    \geq
    \frac{1}{\sqrt{2}} \Bigg( -\sigma_{w} \sqrt{k}\frac{1}{\sigma_w}(1-\frac{1}{\sigma_w})^{k-1} + \sigma_{w} \frac{\sqrt{k} \frac{1}{\sqrt{2\pi}\sigma_{w}}}{{\sqrt{(k-1)\frac{1}{\sqrt{2\pi}\sigma_{w}}+1}}} \Bigg) \mathbb{E}\Big[\norm{z^{(d)}_{\bot}}\Big]
\end{gather}

and with some simplifications: 
\begin{gather}
\label{eq:sixth_2_statement_to_recurrence}
    \mathbb{E}_{W^{(d)}}\Bigg[\norm{\frac{dz}{dt}^{(d+1)}_{\bot}} \Bigg] 
    \geq
    \frac{1}{\sqrt{2}} \Bigg( - \sqrt{k}(1-\frac{1}{\sigma_w})^{k-1} + (2\pi)^{-1/4}  \frac{\sqrt{k \sigma_w}}{{\sqrt{(k-1)+\sqrt{2\pi}\sigma_w}}} \Bigg) \mathbb{E}\Big[\norm{z^{(d)}_{\bot}}\Big]
\end{gather}

Now, we want to roll back Eq. \ref{eq:sixth_2_statement_to_recurrence} to arrive at the inputs. To do this, we replace the expectation term on the right-hand-side by: 

\begin{equation}
    \label{eq:help_help}
    \mathbb{E}\Big[\norm{z^{(d)}_{\bot}}\Big] = \mathbb{E}\Bigg[\norm{\int_t \frac{dz}{dt}^{(d)}_{\bot} dt} \Bigg] 
\end{equation}

\begin{proposition}
\label{prop:1}
Let $f:\mathbb{R} \rightarrow S$, be an integratable function, on Banach space S. Then the following holds:
\begin{equation}
  \int_t \norm{f(t)} dt \geq \norm{\int_t f(t) dt}.
\end{equation}
\end{proposition}
\begin{proof}
let $x = \int_t f(t) dt ~\in S$, and $\Lambda \in S^{*}$ with $\norm{\Lambda} = 1$. Then we have:
\begin{equation}
    \Lambda x = \int_t \Lambda f(t) dt \leq \int_t \norm{\Lambda}_{S^*} \norm{f(t)}_{S} dt = \int_t \norm{f(t)} dt.
\end{equation}
Now based on Hahn-Banach we have: $ \norm{x} \leq \int_t \norm{f(t)} dt$.
\end{proof}

Based on Proposition \ref{prop:1} and Eq. \ref{eq:help_help} we have:

\begin{equation}
    \mathbb{E}\Bigg[\norm{\int_t \frac{dz}{dt}^{(d)}_{\bot} dt} \Bigg] \geq \mathbb{E}\Bigg[\int_t \norm{\frac{dz}{dt}^{(d)}_{\bot} } dt\Bigg] = l(z^{(d)}_{\bot} (t)).
\end{equation}

Now by By recursively rolling out the the expression of Eq. \ref{eq:sixth_2_statement_to_recurrence} to arrive at input, $I(t)$ and denoting $c_1 = \frac{l(I_{\bot}(t))}{l(I(t))}$, we have:

\begin{gather}
\label{eq:seventh_statement_to_recurrence}
    \mathbb{E}_{W^{(d)}}\Bigg[\norm{\frac{dz}{dt}^{(d+1)}_{\bot}} \Bigg] 
    \geq
    \Bigg(\frac{1}{\sqrt{2}} \Bigg( - \sqrt{k}(1-\frac{1}{\sigma_w})^{k-1} + (2\pi)^{-1/4}  \frac{\sqrt{k \sigma_w}}{{\sqrt{(k-1)+\sqrt{2\pi}\sigma_w}}} \Bigg)\Bigg)^{d} c_1 l(I(t))
\end{gather}

Finally, the asymptotic form of the bound, and considering $c_1 \approx 1$ for input trajectories which are orthogonal to their successive time-points gives us:

\begin{gather}
\label{eq:eights_statement_to_recurrence}
    \mathbb{E}_{W^{(d)}}\Bigg[\norm{\frac{dz}{dt}^{(d+1)}_{\bot}} \Bigg] 
    \geq
    O\Bigg(\frac{\sqrt{k \sigma_{w}}}{\sqrt{k+ \sigma_w}}\Bigg)^{d} \norm{I(t)}.
\end{gather}

Eq. \ref{eq:eights_statement_to_recurrence} shows the lower bound for every infinitesimal fraction of the length of the hidden state (in principle components state, $z$, for a neural ODE architecture. consequently, the overall trajectory length is bounded by:

\begin{gather}
\label{eq:ninth_statement_to_recurrence}
    \mathbb{E}\Bigg[ l(z^{(d)}(t))\Bigg] 
    \geq
    O\Big(\frac{\sqrt{k\sigma_w}}{\sqrt{k+\sigma_w}}\Big)^{d \times L} l(I(t)),
\end{gather}
with $L$ being the number ODE steps. Finally we consider the non-zero bias case:

As stated in the Notations section, network parameters are set by $W^{(d)} \sim \mathcal{N}(0,\sigma^{2}_{w}/k)$ and bias vectors as $b^{(d)} \sim \mathcal{N}(0,\sigma^{2}_{b})$. Therefore, the variance of the $h^{(d+1)}_{i}$ will be smaller than $\sigma^{2}_{w} + \sigma^{2}_{b}$. Therefore we have \cite{raghu2017expressive}: 
\begin{equation}
    p = \mathbb{P}(|h^{(d+1)}_{i} |< 1) \geq  \frac{1}{\sqrt{2\pi} \sqrt{\sigma^{2}_{w} + \sigma^{2}_{b}}}
\end{equation}

By replacing this into Eq. \ref{eq:pp}, and simplify further we get:

\begin{gather}
\label{eq:last_statement_to_recurrence}
    \mathbb{E}\Bigg[ l(z^{(d)}(t))\Bigg] 
    \geq
    O\Bigg(\frac{\sigma_w\sqrt{k}}{\sqrt{ \sigma^{2}_{w} + \sigma^{2}_{b} + k \sqrt{\sigma^{2}_{w} + \sigma^{2}_{b}}}}\Bigg)^{d \times L} l(I(t)),
\end{gather}

the main statement of Theorem \ref{theorem:neural_ode} for Neural ODEs is obtained.

\textbf{Deriving the trajectory length lower-bound for CT-RNNs   } For a successive layer $d+1$ of a CT-RNN the gradient between the states at $t+\delta t$ and $t$, $x^{d+1}(t+\delta t)$ and $x^{d+1}(t)$ is determined by:

\begin{equation}
 \frac{dx}{dt}^{(d+1)} = - w^{(d+1)}_{\tau}x^{(d+1)} + f(h^{(d)}),~~~~~~~h^{(d)} = W^{(d)}x^{(d)}+b^{(d)}.
\end{equation}

With $W^{(d+1)}_{\tau}$ standing for the parameter vector $\frac{1}{\tau^{(d+1)}}$, which is conditioned to be strictly positive. Accordingly, for the latent representation (the first two principle components of the hidden state $x^{(d+1)}$), which is denoted by $z^{(d+1)}(t)$, this gradient can be determined by:

\begin{equation}
 \frac{dz}{dt}^{(d+1)} = - W^{(d+1)}_{\tau}z^{(d+1)} + f(h^{(d)}),~~~~~~~h^{(d)} = W^{(d)}z^{(d)}+b^{(d)}
\end{equation}

An explicit Euler discretization of this ODE gives us: 
\begin{equation}
    z^{(d+1)}(t+\delta t) = (1- \delta t W^{(d+1)}_{\tau}) z^{(d+1)} + \delta t f(h^{(d)}),~~~~~~~h^{(d)} = W^{(d)}z^{(d)}+b^{(d)}.
\end{equation}

the same discretization model for Neural ODEs gives us:

\begin{equation}
    z^{(d+1)}(t+\delta t) = z^{(d+1)} + \delta t f(h^{(d)}),~~~~~~~h^{(d)} = W^{(d)}z^{(d)}+b^{(d)}.
\end{equation}

The difference between the two representations is only a $- \delta t W^{(d+1)}_{\tau}$ term before $z^{(d+1)}$, which consists of $W^{(d+1)}_{\tau}$ that is a strictly positive random variable sampled from a folded normal distribution $\mathcal{N}(|x|;\mu_Y,\sigma_Y)$, with mean $\mu_Y = \sigma \sqrt{\frac{2}{\pi}} e^{(-\mu^2/2\sigma^2)} -\mu (1-2 \Phi(\frac{\mu}{\sigma}))$ and variance $\sigma^{2}_{Y} = \mu^2 + \sigma^2 - \mu^{2}_{Y}$ \cite{tsagris2014folded}. $\mu$ and $\sigma$ are the mean and variance of the normal distribution over random variable $x$, and $\Phi$ is a normal cumulative distribution function. For a zero-mean normal distribution with variance of $\sigma^{2}_{b}$, we get: 

\begin{equation}
    \mathcal{N}(|W_{\tau}|; \sigma_{b} \sqrt{\frac{2}{\pi}},~(1-\frac{2}{\pi})\sigma^{2}_{b}).
\end{equation}

Accordingly, we approximate the lower-bound for the CT-RNNs, with the simplified asymptotic form of:

\begin{equation}
    \mathbb{E}\Bigg[ l(z^{(d)}(t))\Bigg] 
    \geq
    O\Bigg(\frac{(\sigma_w-\sigma_b)\sqrt{k}}{\sqrt{ \sigma^{2}_{w} + \sigma^{2}_{b} + k \sqrt{\sigma^{2}_{w} + \sigma^{2}_{b}}}}\Bigg)^{d \times L} l(I(t)),
\end{equation}

This gives of the statement of the theorem for CT-RNNs.

\end{proof}

\section*{Proof of Theorem \ref{theorem:LTC}}

\textbf{Distribution of parameters of LTCs   }

The Weight matrix for each layer d of the form $W^{(d)} \sim \mathcal{N}(0,\sigma^{2}_{w}/k)$. The bias vectors as $b^{(d)} \sim \mathcal{N}(0,\sigma^{2}_{b})$. The vector parameter $W^{(d+1)}_{\tau}$ is strictly positive and it is sampled from a folded normal distribution \cite{tsagris2014folded} $ \mathcal{N}(|W_{\tau}|; \sigma_{b} \sqrt{\frac{2}{\pi}},~(1-\frac{2}{\pi})\sigma^{2}_{b})$. The parameter stands for the inverse of the time-constant of neurons, $\frac{1}{\tau^{(d+1)}}$
The parameter $A^{(d)}$ is a weight matrix sampled from $\sim \mathcal{N}(0,\sigma^{2}_{w}/k)$.

\begin{proof}
For a successive layer $d+1$ of an LTC network, the gradient between the states at $t+\delta t$ and $t$, $x^{d+1}(t+\delta t)$ and $x^{d+1}(t)$ is determined by:

\begin{equation}
 \frac{dx}{dt}^{(d+1)} = - (w^{(d+1)}_{\tau} + f(h^{(d)}))x^{(d+1)} + A^{(d)}f(h^{(d)}),~~~~~~~h^{(d)} = W^{(d)}x^{(d)}+b^{(d)}.
\end{equation}

Accordingly, for the latent representation (the first two principle components of the hidden state $x^{(d+1)}$), which is denoted by $z^{(d+1)}(t)$, this gradient can be determined by:

\begin{equation}
\label{eq:ltc_recurrence_first}
 \frac{dz}{dt}^{(d+1)} = - (w^{(d+1)}_{\tau} + f(h^{(d)}))z^{(d+1)} + A^{(d)}f(h^{(d)}),~~~~~~~h^{(d)} = W^{(d)}z^{(d)}+b^{(d)}.
\end{equation}

We first take the expectation of norms from both side of Eq. \ref{eq:ltc_recurrence_first}, while similar to Eq. \ref{eq:similar_1} and based on Lemma \ref{lem:norm_of_projections}, we decompose the expectation over parallel and orthogonal components of the weight matrix $W^{(d)}$ as follows:

\begin{equation}
    \mathbb{E}_{W^{(d)}}\Bigg[\norm{\frac{dz}{dt}^{(d+1)}}\Bigg] = \mathbb{E}_{W^{(d)}_{\parallel}}\mathbb{E}_{W^{(d)}_{\bot}}\Big[\norm{- (w^{(d+1)}_{\tau} + f(h^{(d)}))z^{(d+1)} + A^{(d)}f(h^{(d)})}\Big].
\end{equation}

We can now derive the following inequality for the norms of difference versus difference of norms as follows:

\begin{gather}
    \mathbb{E}_{W^{(d)}}\Bigg[\norm{\frac{dz}{dt}^{(d+1)}}\Bigg] =
    \\
    \mathbb{E}_{W^{(d)}_{\parallel}}\mathbb{E}_{W^{(d)}_{\bot}}\Big[\norm{A^{(d)}f(h^{(d)} - (w^{(d+1)}_{\tau} + f(h^{(d)}))z^{(d+1)})}\Big] \geq
    \\
    \mathbb{E}_{W^{(d)}_{\parallel}}\mathbb{E}_{W^{(d)}_{\bot}}\Big[\norm{A^{(d)}f(h^{(d)})} - \norm{(w^{(d+1)}_{\tau} + f(h^{(d)}))z^{(d+1)}}\Big] \geq
    \\
    \mathbb{E}_{W^{(d)}_{\parallel}}\mathbb{E}_{W^{(d)}_{\bot}}\Big[\norm{A^{(d)}f(h^{(d)})}\Big] - \mathbb{E}_{W^{(d)}_{\parallel}}\mathbb{E}_{W^{(d)}_{\bot}}\Big[\norm{(w^{(d+1)}_{\tau} + f(h^{(d)}))z^{(d+1)}}\Big].
    \label{eq:ltc_recurrence_second}
\end{gather}

Let us first focus on the \textbf{right expression} in Eq. \ref{eq:ltc_recurrence_second}. The norm can be split into the norm of products, as follows:
\begin{equation}
\mathbb{E}_{W^{(d)}_{\parallel}}\mathbb{E}_{W^{(d)}_{\bot}}\Big[\norm{(w^{(d+1)}_{\tau} + f(h^{(d)}))}\norm{z^{(d+1)}}\Big].
\end{equation}

Now by conditioning the expectations by the following rule $\mathbb{E}[XY] = \mathbb{E}[X]\mathbb{E}[Y]$, we get:

\begin{equation}
\mathbb{E}_{W^{(d)}_{\parallel}}\mathbb{E}_{W^{(d)}_{\bot}}\Big[\norm{(w^{(d+1)}_{\tau} + f(h^{(d)}))}\Big]\mathbb{E}\Big[\norm{z^{(d+1)}}\Big].
\end{equation}

We determine the set of indices for which $f$ is not saturated and we assume that it is defined by \url{Hard-tanh} activations:
\begin{equation}
    \mathcal{A}_{W^{(d)}_{\parallel}} = \{i : i \in [k], |h_{i}^{(d+1)}| < 1 \}
\end{equation}

Now, if we condition on $W^{(d)}_{\parallel}$, we can replace the first norm by the sum over the non-saturated indices, $\mathcal{A}_{W^{(d)}_{\parallel}}$ as follows:
\begin{equation}
\label{eq:ltc_recurrence_third}
\Resize{13cm}{\mathbb{E}_{W^{(d)}_{\parallel}}\mathbb{E}_{W^{(d)}_{\bot}}\Bigg[\Bigg(\sum_{i \in \mathcal{A}_{W^{(d)}_{\parallel}}} \big( (W^{(d)}_{\bot} + \frac{w^{(d+1)}_{\tau}}{|\mathcal{A}|})_{i}\ z^{(d)} + (W^{(d)}_{\parallel} + \frac{w^{(d+1)}_{\tau}}{|\mathcal{A}|})_{i}\ z^{(d)} \big)^{2}\Bigg)^{1/2}\Bigg]\mathbb{E}\Big[\norm{z^{(d+1)}}\Big].}
\end{equation}

In Eq. \ref{eq:ltc_recurrence_third}, the term $ \frac{w^{(d+1)}_{\tau}}{|\mathcal{A}|}$ determines the average effect of the time-constant weights in the computation of each state which is a constant addition. $|\mathcal{A}|$ is the number of non-saturated states. Now by taking similar steps, from Eq. \ref{eq:expected_global} to Eq. \ref{eq:forth_statement_to_recurrence}, and by applying Lemma \ref{lem:norm_of_projections} to Eq. \ref{eq:ltc_recurrence_third}, we have:

\begin{equation}
\begin{gathered}
\label{eq:ltc_recurrence_forth}
    \mathbb{E}_{W^{(d)}_{\bot}}\Bigg[\Bigg(\sum_{i \in \mathcal{A}_{W^{(d)}_{\parallel}}} \big( (\tensor[^\bot]{W}{_{\bot}}^{(d)}+ \frac{w^{(d+1)}_{\tau}}{|\mathcal{A}_{W^{(d)}_{\parallel}}|})_{i}\ z^{(d)}_{\bot}\big)^{2}\Bigg)^{1/2}\Bigg] \mathbb{E}_{W^{(d)}}\Big[\norm{z^{(d+1)}}\Big]
    \geq 
    \\
    \sqrt{\frac{\sigma^{2}_{w}}{k}  + \frac{\sigma^{2}_{b}}{|\mathcal{A}_{W^{(d)}_{\parallel}}|^2}}\sqrt{2}\frac{\sqrt{2|\mathcal{A}_{W^{(d)}_{\parallel}}|-3}}{2} \mathbb{E}\Big[\norm{z^{(d)}_{\bot}}\Big]\mathbb{E}\Big[\norm{z^{(d+1)}}\Big].
\end{gathered}
\end{equation}

As we selected \url{Hard-tanh} activation functions with $p = \mathbb{P}(|h^{(d+1)}_{i} |< 1)$, and the condition $|\mathcal{A}_{W^{(d)}_{\parallel}}| \geq 2$ we have $\sqrt{2}\frac{\sqrt{2|\mathcal{A}_{W^{(d)}_{\parallel}}|-3}}{2} \geq \frac{1}{\sqrt{2}}\sqrt{|\mathcal{A}_{W^{(d)}_{\parallel}}|}$, and therefore we can simplify further:

\begin{equation}
\begin{gathered}
\label{eq:ltc_recurrence_fifth}
    \mathbb{E}_{W^{(d)}_{\bot}}\Bigg[\Bigg(\sum_{i \in \mathcal{A}_{W^{(d)}_{\parallel}}} \big( (\tensor[^\bot]{W}{_{\bot}}^{(d)}+ \frac{w^{(d+1)}_{\tau}}{|\mathcal{A}_{W^{(d)}_{\parallel}}|})_{i}\ z^{(d)}_{\bot}\big)^{2}\Bigg)^{1/2}\Bigg] \mathbb{E}\Big[\norm{z^{(d+1)}}\Big]
    \geq 
    \\
    \frac{1}{\sqrt{2}}\sqrt{\frac{\sigma^{2}_{w} |\mathcal{A}_{W^{(d)}_{\parallel}}|}{k}  + \underbrace{\frac{\sigma^{2}_{b}}{|\mathcal{A}_{W^{(d)}_{\parallel}}|}}_\text{$< < 1$}} \mathbb{E}\Big[\norm{z^{(d)}_{\bot}}\Big]~\mathbb{E}\Big[\norm{z^{(d+1)}}\Big].
\end{gathered}
\end{equation}

Finally, we have:

\begin{gather}
\label{eq:ltc_recurrence_sixth}
\mathbb{E}_{W^{(d)}_{\bot}}\Big[\norm{(w^{(d+1)}_{\tau} + f(h^{(d)}))}\Big]\mathbb{E}\Big[\norm{z^{(d+1)}}\Big]
    \geq 
    \frac{1}{\sqrt{2}} \frac{\sigma_{w}}{\sqrt{k}} \sqrt{|\mathcal{A}_{W^{(d)}_{\parallel}}|} \mathbb{E}\Big[\norm{z^{(d)}_{\bot}}\Big]~\mathbb{E}\Big[\norm{z^{(d+1)}}\Big].
\end{gather}

Now if we take the computational steps from Eq. \ref{eq:fifth_statement_to_recurrence} to \ref{eq:pp}, we obtain the following:
\begin{equation}
\begin{gathered}
\label{eq:ltc_recurrence_seventh}
\mathbb{E}_{W^{(d)}_{\bot}}\Big[\norm{(w^{(d+1)}_{\tau} + f(h^{(d)}))}\Big]\mathbb{E}\Big[\norm{z^{(d+1)}}\Big]
    \geq
    \\
        \frac{1}{\sqrt{2}} \Bigg( -\sigma_{w} \sqrt{k}p(1-p)^{k-1} + \sigma_{w} \frac{\sqrt{k} p}{{\sqrt{(k-1)p+1}}} \Bigg) \mathbb{E}\Big[\norm{z^{(d)}_{\bot}}\Big]~\mathbb{E}\Big[\norm{z^{(d+1)}}\Big].
\end{gathered}
\end{equation}

As stated before, network parameters are set by $W^{(d)} \sim \mathcal{N}(0,\sigma^{2}_{w}/k)$ and bias vectors as $b^{(d)} \sim \mathcal{N}(0,\sigma^{2}_{b})$. Therefore, the variance of the $h^{(d+1)}_{i}$ will be smaller than $\sigma^{2}_{w} + \sigma^{2}_{b}$. Therefore we have \cite{raghu2017expressive}: 
\begin{equation}
    p = \mathbb{P}(|h^{(d+1)}_{i} |< 1) \geq  \frac{1}{\sqrt{2\pi} \sqrt{\sigma^{2}_{w} + \sigma^{2}_{b}}}
\end{equation}

This will give us the following asymptotic bound for the right expression of Eq. \ref{eq:ltc_recurrence_second} as follows:

\begin{equation}
\begin{gathered}
\mathbb{E}_{W^{(d)}_{\bot}}\Big[\norm{(w^{(d+1)}_{\tau} + f(h^{(d)}))}\Big]\mathbb{E}\Big[\norm{z^{(d+1)}}\Big]
    \geq
    \\
O\Bigg(\frac{\sigma_w\sqrt{k}}{\sqrt{ \sigma^{2}_{w} + \sigma^{2}_{b} + k \sqrt{\sigma^{2}_{w} + \sigma^{2}_{b}}}}\Bigg) \mathbb{E}\Big[\norm{z^{(d)}_{\bot}}\Big]~\mathbb{E}\Big[\norm{z^{(d+1)}}\Big]
\end{gathered}
\end{equation}

Now let us work with the \textbf{Left expression} in Eq. \ref{eq:ltc_recurrence_second}:

\begin{gather}
\label{eq:ltc_important_1}
    \mathbb{E}_{W^{(d)}_{\parallel}}\mathbb{E}_{W^{(d)}_{\bot}}\Big[\norm{A^{(d)}f(h^{(d)})}\Big]
\end{gather}

As $A$ serves as a constant, we can take it out of the norm and the expectations. The resulting expectation of the norm, precisely expresses a deep neural network $f$ with \url{Hard-tanh} activations, for which \cite{raghu2017expressive} showed that it can be bound as follows:

\begin{gather}
  |A^{(d)}| \mathbb{E}_{W^{(d)}_{\parallel}}\mathbb{E}_{W^{(d)}_{\bot}}\Big[\norm{f(h^{(d)})}\Big] 
  \geq
  O\Bigg(\frac{\sigma_w\sqrt{k}}{\sqrt{ \sigma^{2}_{w} + \sigma^{2}_{b} + k \sqrt{\sigma^{2}_{w} + \sigma^{2}_{b}}}}\Bigg) |A^{(d)}| \mathbb{E}\Big[\norm{z^{(d)}_{\bot}}\Big]
\end{gather}

And since $A \sim \mathcal{N}(0,\sigma^{2}_{w})$, the bound can be computed as follows: 

\begin{gather}
\label{eq:ltc_important_2}
  \mathbb{E}_{W^{(d)}_{\parallel}}\mathbb{E}_{W^{(d)}_{\bot}}\Big[\norm{A^{(d)}f(h^{(d)})}\Big] 
  \geq
  O\Bigg(\frac{\sigma_{w}^{2} \sqrt{k}}{\sqrt{ \sigma^{2}_{w} + \sigma^{2}_{b} + k \sqrt{\sigma^{2}_{w} + \sigma^{2}_{b}}}}\Bigg) \mathbb{E}\Big[\norm{z^{(d)}_{\bot}}\Big].
\end{gather}

Therefore, for the perpendicular compartments of the gradient of the hidden state, we have:

\begin{equation}
\begin{gathered}
     \mathbb{E}_{W^{(d)}}\Bigg[\norm{\frac{dz}{dt}^{(d+1)}_{\bot}} \Bigg] 
     \geq
  O\Bigg(\frac{\sigma_w\sqrt{k}}{\sqrt{ \sigma^{2}_{w} + \sigma^{2}_{b} + k \sqrt{\sigma^{2}_{w} + \sigma^{2}_{b}}}}\Bigg) \mathbb{E}\Big[\norm{z^{(d)}_{\bot}}\Big]~\mathbb{E}\Big[\norm{z^{(d+1)}}\Big]
  +\\
  O\Bigg(\frac{\sigma_{w}^{2} \sqrt{k}}{\sqrt{ \sigma^{2}_{w} + \sigma^{2}_{b} + k \sqrt{\sigma^{2}_{w} + \sigma^{2}_{b}}}}\Bigg) \mathbb{E}\Big[\norm{z^{(d)}_{\bot}}\Big].
\end{gathered}
\end{equation}

If we simplify further and considering the fact that we are shaping the recurrence for every infinitesimal $\delta t$ of the system's dynamics, we get the following asymptotic bound:

\begin{equation}
\begin{gathered}
     \mathbb{E}_{W^{(d)}}\Bigg[\norm{\frac{dz}{dt}^{(d+1)}_{\bot}} \Bigg] 
     \geq
  O\Bigg(\frac{\sigma_w\sqrt{k}}{\sqrt{ \sigma^{2}_{w} + \sigma^{2}_{b} + k \sqrt{\sigma^{2}_{w} + \sigma^{2}_{b}}}}\Bigg) \mathbb{E}\Big[\norm{z^{(d)}_{\bot}}\Big]~ \Big(\sigma_{w} + \frac{\norm{z^{(d+1)}}}{\min(\delta t, L)}\Big).
\end{gathered}
\end{equation}

Now similar as before, by recursively unrolling the $n$ layer neural network $f$ to reach the input, denoting $c_1 = \frac{l(I_{\bot}(t))}{l(I(t))} \approx 1$, and establishing the bound for an input sequence of length $T$, for a layer $d$ of a network we get:

\begin{equation}
\begin{gathered}
\label{eq:done_ltc}
     \mathbb{E}\Bigg[ l(z^{(d)}(t))\Bigg] 
     \geq
  O\Bigg(\Big(\frac{\sigma_w\sqrt{k}}{\sqrt{ \sigma^{2}_{w} + \sigma^{2}_{b} + k \sqrt{\sigma^{2}_{w} + \sigma^{2}_{b}}}}\Big)^{d \times L} \Big(\sigma_{w} + \frac{\norm{z^{(d)}}}{\min(\delta t, L)}\Big)\Bigg) l(I(t)).
\end{gathered}
\end{equation}

Equation \ref{eq:done_ltc} gives us the statement of the theorem.

\end{proof}

\section{Experimental Setup - Section 6}
Here, we describe the experimental setup for the tasks discussed in Tables \ref{tab:res_32}, \ref{tab:per-time-point_classification}, \ref{tab:per-time-point_classification_2}, and \ref{tab:per-sequence}.

For each experiment we performed a training-validation-test split of 75:10:15 ratio, with the exact ratios depending on the specific dataset. 
After each training epoch the validation metric was evaluated. We kept a backup of the network weights of the configuration that achieved the best validation metric over the whole training process. At the end of the training process, we restored the backed-up weights and evaluated the network on the test-set.
We repeated this procedure for five times with different weight initializations and reported the mean and standard deviation in Tables \ref{tab:res_32}, \ref{tab:per-time-point_classification}, \ref{tab:per-time-point_classification_2}, and \ref{tab:per-sequence}.
Hyper-parameters are shown in Table \ref{tab:hyperparams}.

Each RNN consists of 32 hidden units. As each task requires a different number of output units, the output of the RNNs were fed through a learnable linear layer to project the output to the required dimension.
Note that the objective of our experimental setup is not to build the best predictive models, but to empirically compare the expressive power and generalization abilities of various RNN models.

We implemented all RNN models in \url{TensorFlow 1.14}. For the sake of reproducability, we have submitted all code and data along with our submission and will make them publicly available upon acceptance.

\textbf{ODE solvers}
For simulating the differential equations we used an explicit Euler methods for CT-RNNs, a 4-th order Runge-Kutta method for the Neural ODE as suggested in \cite{chen2018neural}, and our fused ODE solver for LTCs.
All ODE solvers were fixed-step solvers. The time-step is set to 1/6 of the input sampling frequency, i.e., each RNN step consists of 6 ODE solver steps.

\textbf{Hand Gesture Segmentation   } The experiment concerns the temporal segmentation of hand gestures. The dataset consists of seven recordings of individuals performing a sequence of hand gesticulations \cite{wagner2014gesture}. The input features at each time-step are comprised of 32 data points recorded from a motion detection sensor. The output, at each time step, represents one of the five possible hand gestures; rest position, preparation, stroke, hold, and retraction. The objective is to train a classifier to detect hand gestures from the motion data.

We cut each of the seven recordings into overlapping sub-sequences of exactly 32 time-steps. We randomly separated all sub-sequences into non-overlapping training  (75\%), validation (10\%), and test (15\%) sets.
Input features were normalized to have zero mean and unit standard deviation.  We used the categorical classification accuracy as the performance metric.

\textbf{Room Occupancy   } The objective is to detect whether a room is occupied by observations recorded from five physical sensor streams, such as temperature, humidity, and CO2 concentration sensors \cite{candanedo2016accurate}. Input data and binary labels are sampled in one-minute long intervals.

The original dataset consists of a pre-defined training and test set.
We used the binary classification accuracy as the performance metric.
We cut the sequences of each of the two sets into a training and test set of overlapping sub-sequences of exactly 32 time-steps. Note that no item from the test set was leaking into the training set during this process. 
Input features of all data were normalized by the mean and standard deviation of the training set, such that the training set has zero mean and unit standard deviation. We select 10\% of the training set as the validation set. 

\textbf{Human Activity Recognition   } This task involves the recognition of human activities, such as walking, sitting, and standing, from inertial measurements of the user's smartphone \cite{anguita2013public}. Data consists of recordings from 30 volunteers performing activities form six possible categories. Input variables are filtered and are pre-processed to obtain a feature column of 561 items at each time step.

The output variable represents one of six activity categories at each time step. We employed the categorical classification accuracy as our performance metric. The original data is already split into a training and test set and preprocessed by temporal filters. The accelerometer and gyroscope sensor data were transformed into 561 features in total at each time step.
We aligned the sequences of the training and test set into overlapping sub-sequences of exactly 32 time-steps. We select 10\% of the training set as the validation set. 

\textbf{Sequential MNIST   } We also worked with MNIST. While the original MNIST is a computer vision classification problem, we transform the dataset into a sequence classification task. In particular, each sample is encoded as a 28-dimensional time-series of length 28. Moreover, we downscale all input feature to the range [0,1]. We exclude 10\% of the training set and use it as our validation set.

\textbf{Traffic Estimation   } The objective of this experiment is to predict the hourly westbound traffic volume at the US Interstate 94 highway between Minneapolis and St. Paul. Input features consist of weather data and date information such as local time and flags indicating the presence of weekends, national, or regional holidays. The output variable represents the hourly traffic volume.

The original data consists of hourly recordings between October 2012 and October 2018, provided by the Minnesota Department of Transportation and OpenWeatherMap. We selected the seven columns of the data as input features: 1. Flag indicating whether the current day is a holiday, 2. The temperature in Kelvin normalized by annual mean, 3. Amount of rainfall, 4. Amount of snowfall, 5. Cloud coverage in percent, 6.  Flag indicating whether the current day is a weekday, and 7. time of the day preprocessed by a sine function to avoid the discontinuity at midnight. The output variable was normalized to have zero mean and unit standard deviation. We used the mean-squared-error as training loss and evaluation metric.
We split the data into partially overlapping sequences lasting 32 hours.
We randomly separated all sequences into non-overlapping training  (75\%), validation (10\%), and test (15\%) set.

\textbf{Power  } We used the ''Individual household electric power consumption Data Set'' from the UCI machine learning repository \cite{dua2019uci}.
Objective of this task is to predict the hourly active power consumption of a household. Input features are secondary measurement such as the reactive power draw and sub-meterings.
Approximately 1.25\% of all measurements are missing, which we overwrite by the most recent measurement of the same feature. 
We apply a feature-wise whitening normalization and split the dataset into non-overlapping sub-sequences of length 32 time-steps. 
The prediction variable (active power consumption) is also whitened. We use the squared-error as optimization loss and evaluation metric.

\textbf{Ozone Day Prediction   } The objective of task is to forecast ozone days, i.e., days when the local ozone concentration exceeds a critical level. Input features consist of wind, weather, and solar radiation readings.

The original dataset ''Ozone Level Detection Data Set'' was taken from the UCI repository \cite{dua2019uci} consists of daily data points collected by the Texas Commission on Environmental Quality (TCEQ). We split the 6-years period into overlapping sequences of 32 days. 
A day was labeled as ozone day if, for at least 8 hours, the exposure to ozone exceeded 80 parts per billion. Inputs consist of 73 features, including wind, temperature, and solar radiation data. 
The binary predictor variable has a prior of 6.31\%, i.e., expresses a 1:15 imbalance. For the training procedure, we weighted the cross-entropy loss at each day, depending on the label. Labels representing an ozone day were assigned 15 times the weight of a non-ozone day. 
Moreover, we reported the $F_1$-score instead of standard accuracy (higher score is better).

In roughly 27\% of all samples, some of the input features were missing. To not disrupt the continuity of the collected data, we set all missing features to zero. Note that such zeroing of some input features potentially negatively affects the performance of our RNN models compared to non-recurrent approaches and filtering out the missing data. Consequently, ensemble methods and model-based approaches, i.e., methods that leverage domain knowledge \cite{zhang2008forecasting}, can outperform the end-to-end RNNs studied in our experiment. We randomly split all sub-sequences into training (75\%), validation (10\%), and test (15\%) set.

\textbf{Person Activity - 1st Setting   } In this setting we used the "Human Activity" dataset described in \cite{rubanova2019latent}. However, as we use different random seeds for the training-validation-test splitting, and a different input representation, our results are not transferable directly to those obtained by \cite{rubanova2019latent}, in the current setting.

The dataset consists of 25 recordings of various physical activity of human participants, for instance, among others lying down, walking, sitting on the ground. The participants were equipped with four different sensors, each sampling at a period of 211 ms. 

Similar to \cite{rubanova2019latent}, we packed the 11 activity categories into 7 classes. No normalization is applied to the input features. The 25 sequences were split into partially overlapping sub-sequences of length 32 time-steps.

unlike Rubanova et al. \cite{rubanova2019latent}, we represented the input time-series as a 7-dimensional feature vector, where the first 4 entries specified the sensor ID and the last 3 entries the sensor values. Due to the high sampling frequency we discarded all timing information.

The results are reported in Table \ref{tab:per-time-point_classification}.

\textbf{Person Activity - 2nd Setting }
We setup a second experimental setup based on the same dataset as the person activity task above. In contrast to the first setting, we made sure that the training and test sets are equivalent to \cite{rubanova2019latent} in order to be able to directly compare results.  However, we apply the same pre-processing as in our experiment before. In particular, represent the datasets as irregularly sampled in time and dimension using a padding and masking, which results in a 24-dimensional input vector. On the other hand, we discard all time information and feed the input data as described above in the form of a 7-dimensional vector. Note that the data is still the same, just represented in a different format.

Based on the training - test split of \cite{rubanova2019latent} we select 10\% of the training set as our validation set.
Moreover, we train our model for 400 epochs and select the epoch checkpoint which achieved the best results on the validation set.
This model is then selected to be tested on the test set provided by \cite{rubanova2019latent}. Results are reported in Table \ref{tab:per-time-point_classification_2}.

\textbf{Half-Cheetah Kinematic modeling   } This task is inspired by the physics simulation experiment of Chen et al. \cite{rubanova2019latent}, which evaluated how well RNNs are suited to model kinematic dynamics. In our experiment, we collected 25 rollouts of a pre-trained controller for the HalfCheetah-v2 gym environment \cite{brockman2016openai}. Each rollout is composed of a series of 1000 17-dimensional observation vectors generated by the MuJoCo physics engine \cite{todorov2012mujoco}. The task is then to fit the observation space time-series in an autoregressive fashion.To increase the difficulty, we overwrote $5\%$ of the actions produced by the pre-trained controller by random actions. We split the data into training, test, and validation sets by a ratio of 2:2:1. Training loss and test metric were mean squared error (MSE). Results were reported in Table \ref{tab:per-sequence}.

\section{Hyperparameters and Parameter counts - Tables \ref{tab:res_32}, \ref{tab:per-time-point_classification}, and \ref{tab:per-sequence}}

\begin{table}[H]
    \centering
    \caption{Hyperparameters used for the experimental evaluations}
    \begin{tabular}{l|c|l}
    \toprule
         \textbf{Parameter} & \textbf{Value} & \textbf{Description} \\
    \midrule
         Number of hidden units & 32 & \\
         Minibatch size & 16 &  \\
         Learning rate & 0.001 - 0.02 & \\
         ODE-solver step & 1/6 & relative to input sampling period \\
         Optimizer & Adam \cite{kingma2014adam} & \\
         $\beta_1$ & 0.9 & Parameter of Adam\\
         $\beta_2$ & 0.999 & Parameter of Adam \\
         $\hat{\epsilon}$ & 1e-08 & Parameter of Adam \\
         BPTT length & 32 & Backpropagation through time length \\
         & & in time-steps\\
         Validation evaluation interval & 1 & Every x-th epoch the validation \\
         & & metric will be evaluated\\
         Training epochs & 200 & \\
         \bottomrule
    \end{tabular}
    \label{tab:hyperparams}
\end{table}

\begin{table}[H]
    \centering
    \caption{Number of parameters of various RNN model in relation to the RNN width $k$, the number of hidden layers $n$, and the number of decay slots $m$. }
    \begin{tabular}{l|c|l}
    \toprule
         \textbf{Model} & \textbf{Parameter count (asymptotic)} & \textbf{Parameter count (exact)}  \\
         \midrule
         CT-RNN & $O(nk^2) $ & $nk^2+2nk$ \\
         ODE-RNN & $O(nk^2)$ & $nk^2+nk$ \\
         LSTM & $O(nk^2)$ & $4nk^2 + 4nk$ \\
         CT-GRU & $O(m k^2)$ & $2m k^2 + 2mk + k^2+k$ \\
         LTC & $O(nk^2)$ & $4nk^2+3nk$ \\
         \bottomrule
    \end{tabular}
    \label{tab:param_efficiency}
\end{table}

\section{Additional trajectory space representations:}
Trajectory space representation for the results provided can be viewed at: \url{https://www.dropbox.com/s/ly6my34mbvsfi6k/additional_LTC_neurIPS_2020.zip?dl=0}

\section{Trajectory Length results}
\begin{figure}[H]
\centering
\includegraphics[width=1\textwidth]{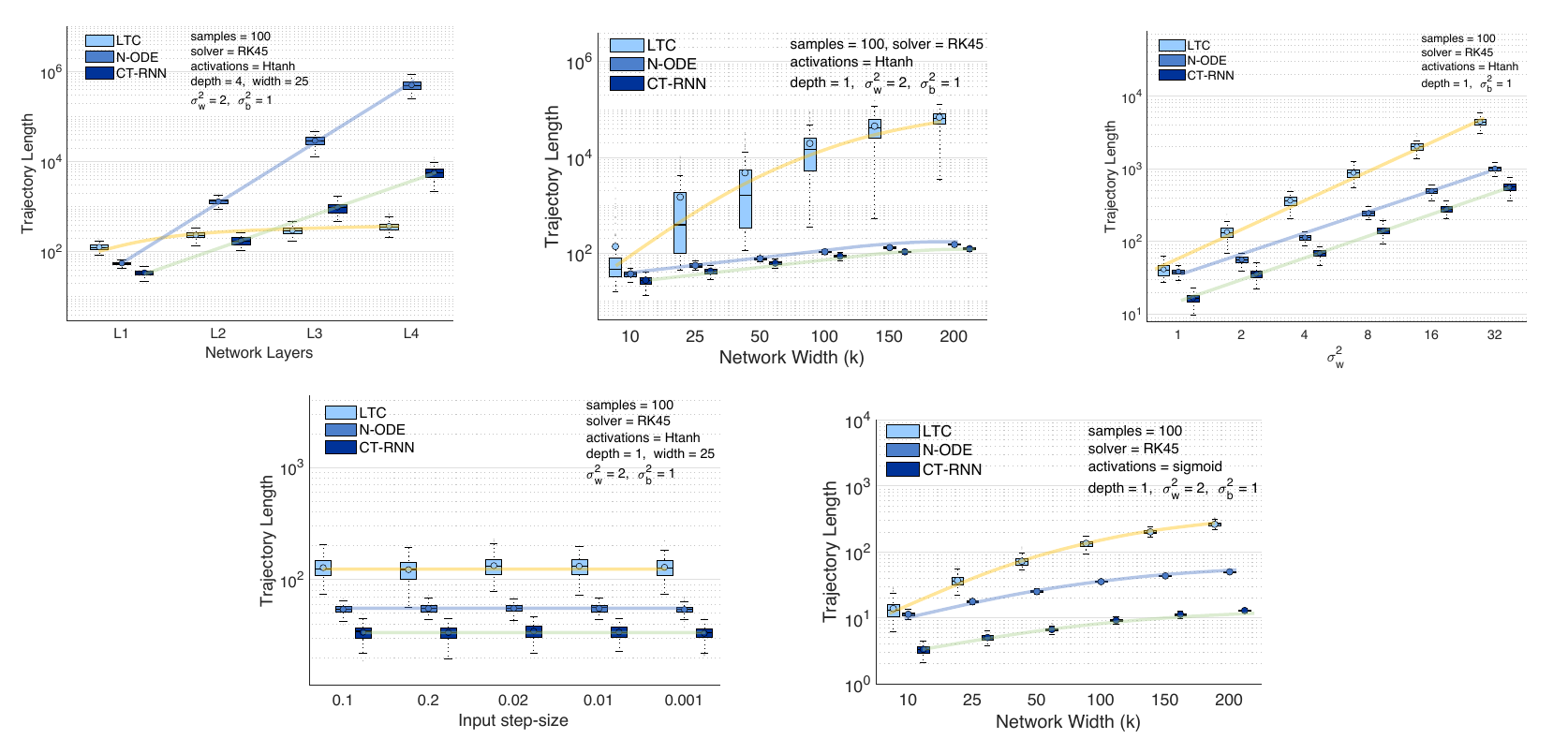}
\caption{Additional trajectory length results.}
\label{fig:trajectory_length_results}
\end{figure}

\section{Code and Data availability}
All code and data are publicly accessible at: \url{https://github.com/raminmh/liquid_time_constant_networks}.

\end{document}